\documentclass[letterpaper]{article} 
\usepackage{aaai2026}  
\usepackage{times}  
\usepackage{helvet}  
\usepackage{courier}  
\usepackage[hyphens]{url}  
\usepackage{graphicx} 
\usepackage{natbib}  
\usepackage{caption} 
\usepackage{algorithm}
\usepackage{algorithmic}
\usepackage{array}
\usepackage{booktabs,makecell, multirow, tabularx,bigstrut}
\usepackage{amssymb}
\usepackage{amsmath,amsfonts,bm}
\usepackage{amsthm}
\usepackage{subcaption}

\usepackage{newfloat}
\usepackage{listings}
\DeclareCaptionStyle{ruled}{labelfont=normalfont,labelsep=colon,strut=off} 
\floatstyle{ruled}
\newfloat{listing}{tb}{lst}{}
\floatname{listing}{Listing}

\newtheorem{assumption}{Assumption}

\newtheorem{lemma}{Lemma}
\newtheorem{theorem}{Theorem} 
\newtheorem{remark}{Remark}

\title{Personalized Federated Learning with Bidirectional Communication Compression via One-Bit Random Sketching}
\author{
    Jiacheng Cheng\textsuperscript{\rm 1}, 
    Xu Zhang\textsuperscript{\rm 2}\thanks{Corresponding author.},
    Guanghui Qiu\textsuperscript{\rm 3},
    Yifang Zhang\textsuperscript{\rm 2,3},
    Yinchuan Li\textsuperscript{\rm 4},
    Kaiyuan Feng\textsuperscript{\rm 5}
}
\affiliations{
    \textsuperscript{\rm 1}School of Automation, Northwestern Polytechnical University, Xi'an, China\\
    \textsuperscript{\rm 2}School of Artificial Intelligence, Xidian University, Xi'an, China\\
    \textsuperscript{\rm 3}Science and Technology on Electronic Information Control Laboratory, China Academy of Electronics and Information Technology, Chengdu, China\\
    \textsuperscript{\rm 4}Knowin AI, Shenzhen, China \\
    \textsuperscript{\rm 5}Academy of Advanced Interdisciplinary Research, Xidian University, Xi'an, China \\
    zhang.xu@xidian.edu.cn 
}

\usepackage{bibentry}

\begin{document}

\maketitle

\begin{abstract}
Federated Learning (FL) enables collaborative training across decentralized data, but faces key challenges of bidirectional communication overhead and client-side data heterogeneity. 
To address communication costs while embracing data heterogeneity, we propose {\bf{pFed1BS}}, a novel {\bf{p}}ersonalized {\bf{fed}}erated learning framework that achieves extreme communication compression through {\bf{one-b}}it random {\bf{s}}ketching. 
In personalized FL, the goal shifts from training a single global model to creating tailored models for each client. In our framework, clients transmit highly compressed one-bit sketches, and the server aggregates and broadcasts a global one-bit consensus. To enable effective personalization, we introduce a sign-based regularizer that guides local models to align with the global consensus while preserving local data characteristics. 
To mitigate the computational burden of random sketching, we employ the Fast Hadamard Transform for efficient projection. Theoretical analysis guarantees that our algorithm converges to a stationary neighborhood of the global potential function. Numerical simulations demonstrate that pFed1BS substantially reduces communication costs while achieving competitive performance compared to advanced communication-efficient FL algorithms.
\end{abstract}


\section{Introduction}
Federated Learning (FL) is an increasingly popular paradigm in deep learning, designed to train machine learning models on distributed client data while preserving privacy \cite{mcmahan2017fedavg, liconvergence}. Although FL performs well under ideal conditions of \textit{i.i.d.} data and unconstrained communication, real-world deployments face two fundamental challenges inherent to its decentralized nature.

First, data across clients is typically \textit{non-i.i.d.}, reflecting diverse user behaviors and environments, which can severely degrade the performance of a single global model. Second, the communication overhead is often prohibitive, as repeatedly transmitting high-dimensional models between a central server and numerous clients is infeasible in bandwidth-limited networks. These challenges are particularly acute in critical application domains such as the massive Internet of Things (IoT), Vehicle-to-Everything (V2X) communications, and remote sensing networks. In these settings, where devices operate under extremely constrained bandwidth, efficient communication is not merely an optimization but a fundamental necessity for the system to be viable.

To overcome the challenge of \textit{non-i.i.d.} data, Personalized Federated Learning (PFL) \cite{t2020personalized, li2020federated} has been proposed. Here, it is crucial to distinguish the goal of "personalization" from merely "addressing data heterogeneity." While the latter often aims to improve a single global model, PFL explicitly seeks to produce bespoke, customized models for individual clients that capture the unique characteristics of their local data. However, most PFL methods still incur significant communication costs by transmitting full-precision, high-dimensional model parameters or updates. This leads to a more precise and challenging research question: \textit{can we design a federated learning algorithm that not only provides high-quality personalized models for each client but also operates efficiently under extreme bidirectional communication constraints?}

To reduce this burden, communication-efficient FL (CEFL) techniques have been developed, employing methods like prototype-learning~\cite{tan2022fedproto}, sparsification \cite{sattler2019robust,liu2023sparse} and quantization \cite{reisizadeh2020fedpaq,mao2022communication}. Among the most aggressive are one-bit compression strategies. For instance, OBDA \cite{zhu2020obda} applies symmetric one-bit quantization for bidirectional communication, OBCSAA \cite{fan2022obcsaa} combines a one-bit compressed sensing uplink with an uncompressed downlink, and zSignFed \cite{tang2024z} stabilizes sign-based compression through a noisy perturbation scheme.

While these methods achieve remarkable compression rates, they primarily focus on training a single global model and thus overlook the critical challenge of data heterogeneity. To systematically analyze the limitations of existing work and motivate our contribution, we provide a comparison of representative algorithms in Table \ref{tab:compression_comparison}. 
The table reveals a clear research gap: no existing framework offers both extreme, bidirectional communication efficiency and native support for personalization.

\begin{table*}[htbp]
	\centering
	\renewcommand{\arraystretch}{1.1}
	\resizebox{\textwidth}{!}{
		\begin{tabular}{c | c c | c c | c}
			\toprule[1.0pt]
			\multirow{2}{*}{\textbf{Algorithm}} & \multicolumn{2}{c|}{\textbf{Upload Compression}} & \multicolumn{2}{c|}{\textbf{Download Compression}} & \multirow{2}{*}{\makecell{\textbf{Personalization} \\ \textbf{Capability}}} \\
			\cline{2-5}
			& \textbf{Dim. Reduction} & \textbf{1-bit Quant.} & \textbf{Dim. Reduction} & \textbf{1-bit Quant.} & \\
			\midrule[0.7pt]
			FedAvg~\cite{mcmahan2017fedavg} & $\times$ & $\times$ & $\times$ & $\times$ & $\times$ \\
			OBDA~\cite{zhu2020obda}         & $\times$ & $\checkmark$ & $\times$ & $\checkmark$ & $\times$ \\
			OBCSAA~\cite{fan2022obcsaa}     & $\checkmark$ & $\checkmark$ & $\times$ & $\times$ & $\times$ \\
			zSignFed~\cite{tang2024z}     & $\times$ & $\checkmark$ & $\times$ & $\times$ & $\times$ \\
			\textbf{pFed1BS}                   & $\checkmark$ & $\checkmark$ & $\checkmark$ & $\checkmark$ & $\checkmark$ \\
			\bottomrule[1.0pt]
		\end{tabular}
	}
	\caption{Comparisons of communication-efficient schemes and personalization Capabilities in Federated Learning Algorithms.}
	\label{tab:compression_comparison}
\end{table*}

In this paper, we bridge this gap by introducing pFed1BS, a novel PFL framework designed for extreme communication constraints based on one-bit random sketching. We formulate a joint optimization problem where clients learn personalized models by minimizing a local loss augmented with a sign-based regularizer. This regularizer encourages alignment with a global consensus vector derived by the server. Critically, our framework achieves bidirectional compression: clients upload one-bit sketches of their local models, and the server broadcasts a compact one-bit consensus vector. This contrasts sharply with prior CEFL methods that either compress only the uplink or require a full-model downlink. 
Figure \ref{fig:method} provides an overview of our proposed framework.

\subsection{Main Contributions}
This paper proposes \textbf{pFed1BS}, as shown in Figure \ref{fig:method}, a personalized federated learning framework designed for settings with extreme communication constraints. Our main contributions are as follows:

\begin{itemize}
	\item We are the first to formulate the problem of personalized learning with one-bit bidirectional communication as a principled joint optimization problem. 
	Our framework defines two coupled objectives: a client-side objective that balances local empirical risk with a novel sign-based regularizer for global alignment, and a server-side objective for optimally aggregating the compressed client signals.
	
	\item We make our framework practical for large-scale models by introducing a highly efficient implementation of the required sketching operations. By leveraging the Fast Hadamard Transform (FHT), we reduce the complexity of the client-side sketching operation from quadratic, $\mathcal{O}(mn)$, to near-linear, $\mathcal{O}(n\log n)$, without performance degradation.
	
	\item We provide the comprehensive convergence analysis for this challenging alternating optimization scheme. 
	We formally prove that pFed1BS converges to a stationary neighborhood of a global potential function, rigorously accounting for the interplay between personalization, local stochastic updates, and errors introduced by one-bit sketching and server aggregation.
	
	\item We conduct extensive experiments on benchmark datasets (MNIST, FMNIST, CIFAR-10, CIFAR-100, and SVHN).
	Our results demonstrate that pFed1BS achieves a superior trade-off between accuracy and communication. Remarkably, pFed1BS matches or exceeds the performance of state-of-the-art one-bit FL algorithms while operating at a fraction of the communication cost and, crucially, providing personalization that they lack.
\end{itemize}

\subsection{Related Works}
Our work is closely related to the following topics:
\paragraph{Personalized Federated Learning.}
To address the challenge of heterogeneous datasets, a rich body of work has emerged in Personalized Federated Learning (PFL).
These approaches can be broadly categorized into local adaptation~\cite{li2020federated,t2020personalized,li2021ditto,zhang2022personalized}, multi-task learning~\cite{smith2017federated,marfoq2021federated}, and architecture-based methods~\cite{arivazhagan2019federated,collins2021exploiting}.
For instance, regularization-based methods like pFedMe~\cite{t2020personalized} and Ditto~\cite{li2021ditto} learn personalized models by augmenting the local objective with a proximal term that regularizes it towards a global model. Architecture-based methods, such as FedRep~\cite{collins2021exploiting}, learn a shared feature representation while personalizing the final model layers. 
More recently, DisPFL~\cite{dai2022dispfl} learns personalized sparse masks for each client. 
However, these advanced PFL methods typically inherit the communication bottlenecks of standard FL, as they still presuppose the transmission of full-precision, high-dimensional model parameters or updates.

\paragraph{Communication-Efficient Federated Learning.}
In a parallel research thrust, numerous methods have been proposed to alleviate the communication burden in FL.  
One prominent approach is update sparsification, where only a fraction of the model update is transmitted, using techniques like Top-k selection~\cite{sattler2019robust} or identifying parameters with high-magnitude changes~\cite{long2024fedsq}.
Another major direction is quantization, which reduces the numerical precision of the transmitted updates~\cite{reisizadeh2020fedpaq,chen2024mixed,mao2022communication}. A critical limitation of these popular techniques, however, is that they typically compress only the uplink (client-to-server) channel, still requiring the server to broadcast a full-precision, high-dimensional model.

More aggressive strategies leverage techniques from signal processing.  
Some works employ Compressed Sensing (CS) to project sparse updates into a low-dimensional subspace~\cite{li2021communication,oh2022communication,oh2023fedvqcs}.
Others explore one-bit quantization schemes, sometimes combined with over-the-air computation in wireless settings, to achieve extreme compression~\cite{zhu2020obda,tang2024z,oh2024communication}. 
While highly efficient, these methods are fundamentally designed to learn a single global model and lack mechanisms to handle data heterogeneity.

pFed1BS uniquely bridges these two disparate lines of research. It is the first work to integrate a bidirectional, one-bit sketching mechanism within a principled PFL formulation, thereby explicitly and simultaneously tackling the dual challenges of communication efficiency and data heterogeneity.

\section{The Proposed Method}
This section presents our proposed method, pFed1BS. 
We begin by formulating the overall optimization framework that governs the collaborative learning process by introducing a sign-based regularizer and random sketching.  
Subsequently, we describe the iterative algorithm, specifying distinct client-server procedures for optimizing this objective. 
Finally, to address the computational bottleneck posed by the high-dimensional random sketching in local training, we introduce an efficient sketching method based on the Fast Hadamard Transform.

\begin{figure*}[!t]
	\centering
	\includegraphics[width=7.0in]{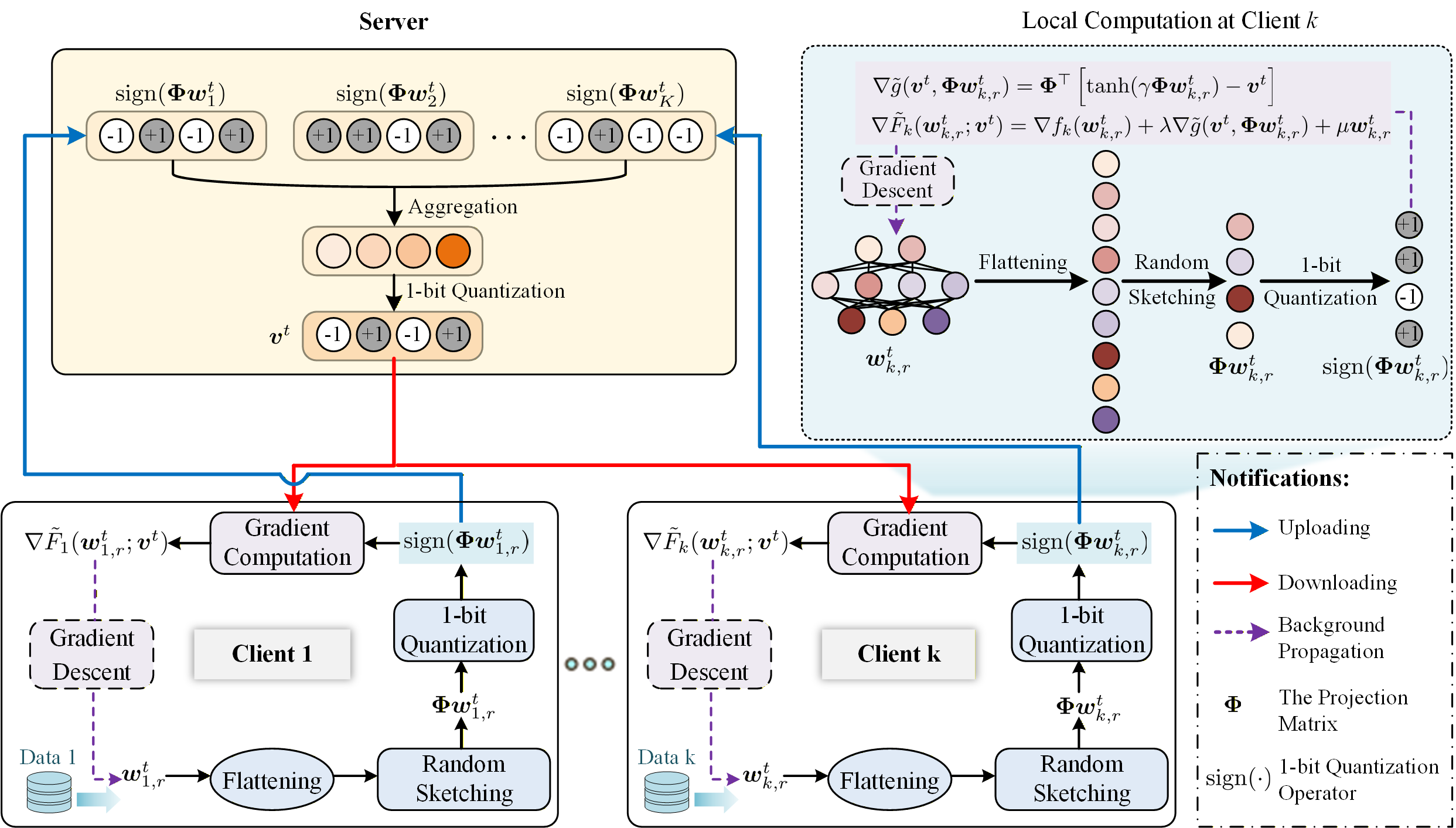}
	\caption{An overview of our proposed framework. 
		At round $t$, each client performs a local update using a sign-based regularizer with the global one-bit vector $\bm{v}^t$. Then each client projects and quantizes the updated local model to one bit vector $\mbox{sign}(\mathbf{\Phi} \bm{w}_k^{t+1})$, and then transmits it to the server. 
		The server aggregates all clients' one-bit vectors to form the next global one-bit vector $\bm{v}^{t+1}$, which is broadcast for the next round.
	}
	\label{fig:method}
\end{figure*}

\subsection{Optimization Framework}
In a federated learning system with $K$ clients, our goal is to move beyond learning a single global model and instead learn a personalized model $\bm{w}_k \in \mathbb{R}^n$ for each client $k \in \{1,\ldots,K\}$.
The system aggregates client contributions using weights $p_k$, which are typically set based on the client's dataset size, e.g., $p_k=N_k/\sum_{i=1}^{K}N_i$
, where $N_k$ is the number of local samples.
The local objective for each client is to minimize the expected loss over its private data distribution $\mathcal{P}_k$ 
\begin{equation}
\label{eq:expected_loss}
f_k(\bm{w}_k) = \mathbb{E}_{\xi_k \sim \mathcal{P}_k } [\hat{f}_k(\bm{w}_k; \xi_k)],
\end{equation}
\noindent where $\xi_k$  is a data sample drawn from $\mathcal{P}_k$ and $\hat{f}_k(\bm{w}_k; \xi_k)$ is the loss function for a single sample $\xi_k$. 
In practice, the true gradient $\nabla f_k(\bm{w}_k)$ is intractable, so we rely on stochastic gradients computed on mini-batches to approximate it.

To drastically reduce communication costs, our core technical idea is to replace the transmission of high-dimensional models $\bm{w}_k$ with low-dimensional one-bit sketches.
Specifically, each client transmits only $\mbox{sign}(\mathbf{\Phi}\bm{w}_k)$, where $\mathbf{\Phi} \in \mathbb{R}^{m \times n}$ is a random projection matrix. 

This radical compression necessitates a new way for the server and clients to interact.
To this end, we introduce a global consensus vector $\bm{v} \in \{\pm 1\}^m$, which the server aggregates from the clients' one-bit sketches.
This vector $\bm{v}$ is then broadcast back to clients and incorporated into their local optimization, which acts as a target, guiding the local model $\bm{w}_k$ to produce a projection $\mathbf{\Phi}\bm{w}_k$ whose signs align with $\bm{v}$.

To enforce this guidance, we introduce a sign-based regularizer, denoted by
$g(\bm{v},\mathbf{\Phi}\bm{w}_k)$, which measures the disagreement between the signs of the projected model $\mathbf{\Phi}\bm{w}_k$ and the global consensus $\bm{v}$.
We define this regularizer using the one-sided $\ell_1$-norm:
\begin{equation}
\begin{aligned}
g(\bm{x}, \bm{y}) = \|[\bm{x} \odot \bm{y}]_{-}\|_{1}, 
\text{~~where~} [{x_i}]_{-} = \min(x_i, 0).
\end{aligned}
\end{equation}
For the client-side objective, where $\bm{v}$ is a sign vector and $\mathbf{\Phi} \bm{w}_k$ is real-valued, this regularizer is equivalent to:
\begin{equation} 
\label{eq:g_func}
{g}(\bm{v},\mathbf{\Phi} \bm{w}_k)=\frac{1}{2}(\|\mathbf{\Phi} \bm{w}_k\|_1-\langle \bm{v},\mathbf{\Phi} \bm{w}_k \rangle).
\end{equation}

Furthermore, to prevent the personalized models $\bm{w}_k$ 
from diverging or growing unbounded during local training due to the absence of global data constraints, we introduce an $\ell_2$ penalty term.

This leads to the following regularized client objective:
\begin{equation}
\label{eq:client_obj_full}
F_k(\bm{w}_k; \bm{v}) = f_k(\bm{w}_k) + \lambda g(\bm{v}, \mathbf{\Phi} \bm{w}_k) +\frac{\mu}{2} \|\bm{w}_k\|_2^2,
\end{equation}
where $\lambda$ controls the strength of the sign alignment and $\mu$ controls the norm of model parameters.

This form, however, is non-smooth due to the $\ell_1$-norm in \eqref{eq:g_func}, posing a challenge for gradient-based optimization. 
To enable gradient-based optimization, we employ a continuously differentiable approximation for the $\ell_1$-norm.
A standard approach is to approximate $\|\bm{z}\|_1$ with  $h_{\gamma}(\bm{z})$, where $h_\gamma(\bm{z})=\frac{1}{\gamma}\sum_{i=1}^{m}\log(\cosh(\gamma z_i))$.
So we obtain a smoothed regularizer 
\begin{equation}
\tilde{g}(\bm{v}, \mathbf{\Phi}\bm{w}_k) = h_\gamma(\mathbf{\Phi}\bm{w}_k) - \langle \bm{v}, \mathbf{\Phi}\bm{w}_k \rangle,
\end{equation}
and a smoothed client-side objective $\tilde{F}_{k}(\bm{w}_k; \bm{v})$
\begin{multline}
\label{eq:smoothed_obj}
\tilde{F}_{k}(\bm{w}_k; \bm{v}) 
= f_k(\bm{w}_k) + \lambda \tilde{g}(\bm{v}, \mathbf{\Phi}\bm{w}_k) +\frac{\mu}{2} \|\bm{w}_k\|_2^2 \\
= f_k(\bm{w}_k) + \lambda \left(h_\gamma(\mathbf{\Phi}\bm{w}_k) - \langle \bm{v}, \mathbf{\Phi}\bm{w}_k \rangle \right) +\frac{\mu}{2} \|\bm{w}_k\|_2^2,
\end{multline}
where the factor of $\frac{1}{2}$ can be absorbed into the hyperparameter $\lambda$. 

The gradient of the smoothed penalty term with respect to $\bm{w}_k$ is then given by:
\begin{equation} 
\label{eq:smooth_grad}
\nabla \tilde{g}(\bm{v}, \mathbf{\Phi}\bm{w}_k) = \mathbf{\Phi}^\top \left( \tanh(\gamma \mathbf{\Phi}\bm{w}_k) - \bm{v} \right).
\end{equation}
As the smoothing parameter $\gamma \to \infty$, the term $\tanh(\gamma \mathbf{\Phi}\bm{w}_k) \approx \mbox{sign} (\mathbf{\Phi}\bm{w}_k)$.
Therefore, the gradient effectively penalizes the misalignment between the signs of the projected local model and the global consensus vector $\bm{v}$, driving the local updates towards alignment.

Having defined how each client utilizes the global vector $\bm{v}$, we now address how the server generates it.
At the end of each round, the server receives a one-bit sketch $\bm{z}_k = \mbox{sign}(\mathbf{\Phi}\bm{w}_k)$ from each participating client $k$. 
The server's task is to aggregate these sketches into a new consensus vector $\bm{v}^{t+1}$ that best represents the collective information.

We formulate this as an optimization problem where the server seeks to find a vector $\bm{v} \in \{\pm 1\}^m$ that minimizes the total weighted disagreement with the received client sketches:
\begin{equation}
\label{eq:server_obj}
\min_{\bm{v} \in \{\pm 1\}^m} \sum_{k=1}^{K} p_k \tilde{g}(\bm{v}, \bm{z}_k).
\end{equation}

Overall, we formulate a bilevel optimization problem as follows
\begin{align}
\label{eq:objective}
\textbf{Server:} \quad &\min_{\bm{v} \in \{\pm 1\}^m} \sum_{k=1}^{K} p_k  \tilde{g}(\bm{v}, \mbox{sign}(\mathbf{\Phi} \bm{w}_k^\star(\bm{v}))) \nonumber\\
\textbf{Clients:} \quad & \bm{w}_k^\star(\bm{v}) \in  \arg \min_{\bm{w}_k \in \mathbb{R}^n} F_k(\bm{w}_k; \bm{v}) = f_k(\bm{w}_k) \nonumber\\ & \qquad \qquad + \lambda  \tilde{g}(\bm{v}, \mathbf{\Phi} \bm{w}_k) 
+\frac{\mu}{2} \|\bm{w}_k\|_2^2,
\end{align} 
where $\bm{w}_k^\star(\bm{v})$ denotes the optimal solution of the $k$-th low-level problem for a given upper-level variable $\bm{v}$.

\subsection{Algorithm}
To solve the joint optimization problem defined in Eq.~\eqref{eq:objective}, we propose an alternating optimization scheme named pFed1BS. 
The core idea is to iteratively perform (i) local optimization on the client side to update the personalized models $\bm{w}_k$ and (ii) global aggregation on the server side to update the consensus vector $\bm{v}$. 
The overall procedure is presented in Algorithm \ref{alg:main}. 
We describe the specifics of each component below.

\begin{algorithm}[t!]
	\caption{pFed1BS: Personalized Federated Learning via One-Bit Random Sketching}
	\label{alg:main}
	\begin{algorithmic}[1]
		\STATE {\bfseries Input:} Total rounds $T$, local steps $R$, learning rate $\eta$, regularization hyperparameters $\lambda$, $\mu$
		\STATE {\bfseries Server Initializes:} Model $\bm{w}^0$, random seed $I$. Broadcasts $I$ to all clients.
		Initializes $\bm{v}^0=\bm{0}$.
		\FOR{ $t = 0$ to $T-1$ }
		\FOR{$k=1$ to $K$  \textbf{in parallel}}
		\STATE $ \bm{z}_k^{t+1}, \bm{w}_{k}^{t+1} \leftarrow$ \texttt{ClientUpdate}$(k, \bm{w}_k^{t}, \bm{v}^{t})$
		\ENDFOR
		\STATE Random sample a subset of clients $\mathcal{S}^t$
		\STATE \textbf{Aggregate signs:} $\bm{v}^{t+1} = \mbox{sign} \left(\sum_{k\in \mathcal{S}^{t}} p_k \bm{z}_k^{t+1}) \right)$
		\ENDFOR
		\vspace{2mm}
		\STATE {\bfseries Function} \texttt{ClientUpdate}$(k, \bm{w}_k^t, \bm{v}^t)$:
		\STATE  {\bfseries Client $k$ Initializes:} $\bm{w}_{k,0}^{t+1} = \bm{w}_{k}^{t}$.
		\FOR{$r = 0$ to $R - 1$}
		\STATE  Sample a mini-batch $\mathcal{B}_{k,r}$ from data distribution $\mathcal{P}_{k}$
		\STATE  Compute gradient: 
		$\nabla \hat{f}_k(\bm{w}_{k,r}^{t+1}; \mathcal{B}_{k,r}) =\frac{1}{|\mathcal{B}_{k,r}|}\sum_{\xi_k \in \mathcal{B}_{k,r}} \nabla \hat{f}_k (\bm{w}_{k,r}^{t+1}; \xi_k) $
		\STATE Compute regularization subgradient: 
		$\nabla \tilde{g}(\bm{v}^t, \mathbf{\Phi}\bm{w}_{k,r}^{t+1}) = \mathbf{\Phi}^{\top}(\tanh(\gamma \mathbf{\Phi}\bm{w}_{k,r}^{t+1})-\bm{v}^t)$
		\STATE Update local model: 
		$\bm{w}_{k,r+1}^{t+1}\leftarrow \bm{w}_{k,r}^{t+1} - \eta(
		\nabla \hat{f}_k(\bm{w}_{k,r}^{t+1}; \mathcal{B}_{k,r}) + \lambda \nabla \tilde{g}(\bm{v}^{t}, \mathbf{\Phi}\bm{w}_{k,r}^{t+1}) +\mu \bm{w}_{k,r}^{t+1})$
		\ENDFOR
		\STATE {\bf{return}}  $\mbox{sign}(\mathbf{\Phi} \bm{w}_{k,R}^{k+1})$, $\bm{w}_{k,R}^{k+1}$
	\end{algorithmic}
\end{algorithm}

At the start of round $t$, each participating client $k \in \mathcal{S}^t$ receives the global consensus vector $\bm{v}^t$.
The client's goal is to update its local model $\bm{w}_k^t$ by approximately minimizing its smoothed objective $\tilde{F}_k(\bm{w}_k; \bm{v}^t)$ (from Eq.~\eqref{eq:smoothed_obj}). 
As shown in Line 16 of Algorithm \ref{alg:main}, the client performs $R$ steps of stochastic gradient descent. 
For local step $r$, the update is
\begin{equation}
\label{eq:client_update_rule}
\bm{w}_{k,r+1} = \bm{w}_{k,r} - \eta \nabla \tilde{F}_{k}(\bm{w}_{k,r}; \bm{v}^t).
\end{equation}
The gradient $\nabla \tilde{F}_{k}$ is composed of the standard task gradient and our regularization terms:
\begin{align}
\label{eq:client_modified_gradient}
\nabla \tilde{F}_k(\bm{w}_{k,r}^{t}; \bm{v}^t)
=& \nabla f_k(\bm{w}_{k,r}^{t}) + \lambda \mathbf{\Phi}^{\top} \left[\tanh(\gamma \mathbf{\Phi} \bm{w}_{k,r}^{t}) - \bm{v}^{t} \right] \nonumber\\
&+\mu \bm{w}_{k,r}^t.
\end{align}
Since the calculation of $\nabla f_k(\bm{w}_{k,r}^{t})$ is intractable, we use the mean over a mini-batch of data $\mathcal{B}_{k,r}$ from $\mathcal{P}_{k}$
\begin{equation}
\nabla \hat{f}_k(\bm{w}_{k,r}^{t+1}; \mathcal{B}_{k,r}) =\frac{1}{|\mathcal{B}_{k,r}|}\sum_{\xi_k \in \mathcal{B}_{k,r}} \nabla \hat{f}_k(\bm{w}_{k,r}^{t+1}; \xi_k). 
\end{equation}

After $R$ steps, the client transmits its new one-bit sketch $\bm{z}_k^{t+1} = \mathrm{sign}(\mathbf{\Phi}\bm{w}_k^{t+1})$ to the server. Upon receiving the sketches $\{\bm{z}_k^{t+1}\}_{k \in \mathcal{S}^t}$, the server generates the next consensus vector $\bm{v}^{t+1}$. 
As defined in Eq.~\eqref{eq:obj}, the server aims to solve:
\begin{equation}
\label{eq:server_obj_sampled}
\min_{\bm{v} \in \{\pm 1\}^m} \sum_{k \in \mathcal{S}^t} p_k g(\bm{v}, \bm{z}_k^{t+1}).
\end{equation}
Crucially, this discrete optimization problem admits an exact, closed-form solution. The following lemma states that the optimal aggregation is a simple weighted majority vote.
\begin{lemma}[Optimal Server Aggregation]
	\label{lem:server_optimality}
	The unique minimizer of the server objective in Eq.~\eqref{eq:server_obj_sampled} is given by:
	\begin{equation}
	\bm{v}^* = \mbox{\rm sign} \left(\sum_{k \in \mathcal{S}^t} p_k \bm{z}_k^{t+1} \right).
	\end{equation}
\end{lemma}

This result is a straightforward but significant application of optimization principles, ensuring that our aggregation step (Line 8 in Algorithm 1) is not a heuristic but is guaranteed to be optimal given the information available to the server.

\subsection{Efficient Projection via Fast Hadamard Transform} 
A naive implementation of the projection $\mathbf{\Phi} \bm{w}$ 
using a dense Gaussian matrix requires $\mathcal{O}(mn)$ computation and memory, which causes a computational burden for large models ($n \gg 10^6$). 
To ensure scalability, we employ a structured projection based on the Subsampled Randomized Hadamard Transform (SRHT)~\cite{zhang2010compressed}, which reduces the complexity to $\mathcal{O}(n\log n)$.

\begin{figure*}[h]
	\centering
	\includegraphics[width=7.0in]{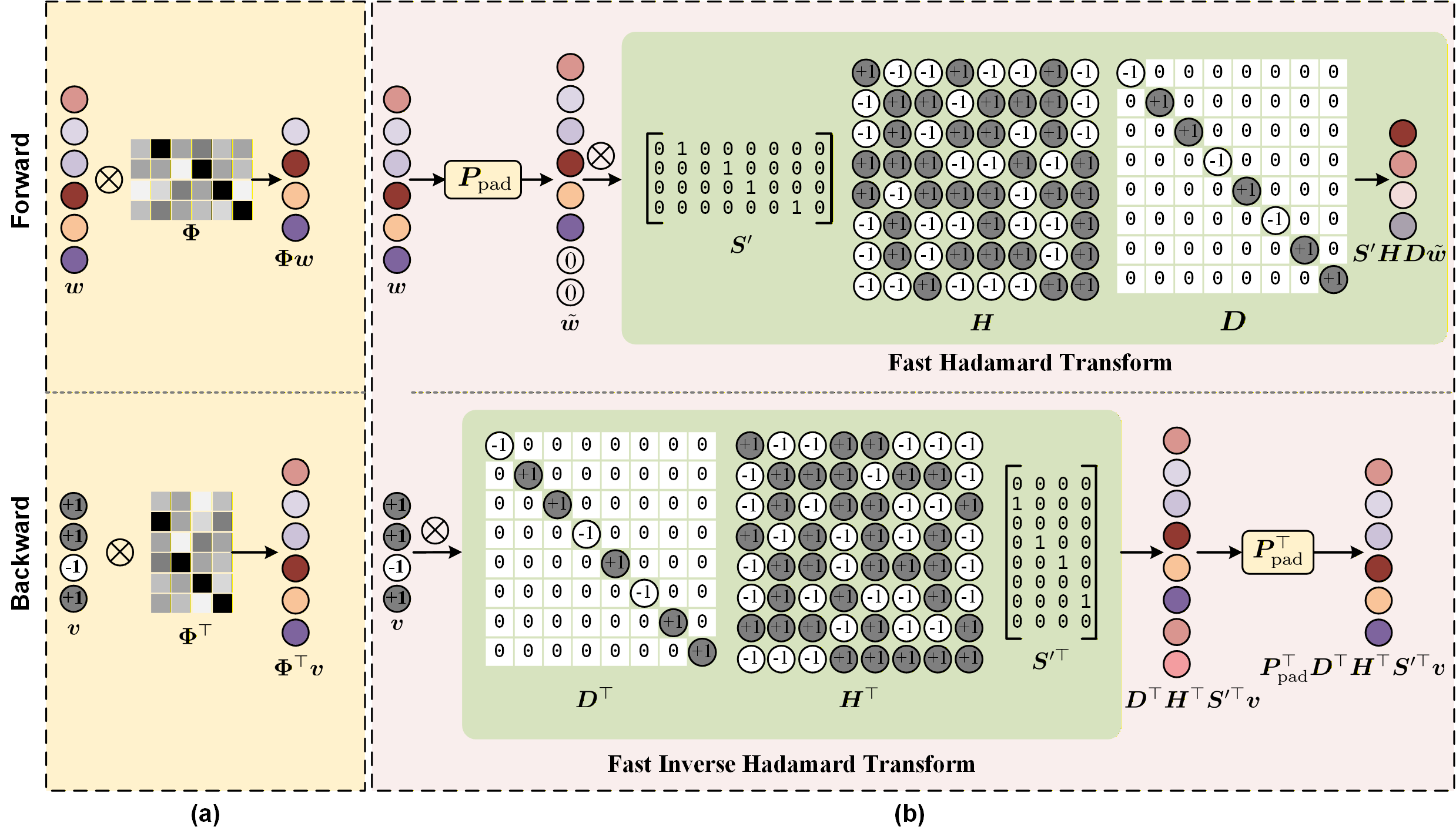}
	\caption{Comparison between (a) a dense random projection and (b) our efficient structured projection. The structured projection sequentially applies element-wise random sign flips ($\bm{D}$), a Fast Hadamard Transform ($\bm{H}$), and random subsampling ($\bm{S}$).}
	\label{fig:fht}
\end{figure*}

\paragraph{Forward Projection.} 
The forward projection $\bm{z}=\mathbf{\Phi} \bm{w} \in \mathbb{R}^{m}$ is computed through a sequence of efficient operations, as illustrated in Figure~\ref{fig:fht} (b) (Up).
Let $n$ denote the model dimension and $m \ll n$ be the target dimension.
First, the input vector $\bm{w} \in \mathbb{R}^n$  is zero-padded to the next power-of-two dimension, $n' = 2^{\lceil \log_2 n \rceil}$.
This can be represented as:
\begin{equation} 
\tilde{\bm{w}} = \bm{P}_{\text{pad}} \bm{w} =\begin{bmatrix} \bm{w} \\ \bm{0} \end{bmatrix},
\end{equation} 
where $\bm{P}_{\text{pad}}$ can be viewed as an $n'\times n$  matrix consisting of the identity matrix $\bm{I}_n$ stacked on top of an $(n'-n) \times n$ zero matrix, and $\tilde{\bm{w}} \in \mathbb{R}^{n'}$.
This step ensures compatibility with the Fast Hadamard Transform (FHT), which requires input lengths that are powers of 2.

Then the structured random projection is performed as follows 
\begin{equation}
\mathbf{\Phi} \bm{w} = \bm{S}' \bm{H} \bm{D} \tilde{\bm{w}},
\end{equation}
where $\bm{D}\in \mathbb{R}^{n' \times n'}$  is a diagonal matrix where each diagonal entry $\bm{D}_{ii}$ is an independent random variable drawn uniformly from $\{-1,+1\}$,
$\bm{H} \in \mathbb{R}^{n' \times n'}$ is the normalized Walsh-Hadamard matrix, and $\bm{S}'=\sqrt{\frac{n'}{m}}\bm{S} \in \mathbb{R}^{m \times n'}$, where $S$ is a subsampling matrix that uniformly selects $m$ rows at random from the $n'$-dimensional identity matrix.
This structured operator avoids forming any dense matrices. The total computational complexity is dominated by the FHT, resulting in an efficient $\mathcal{O}(n\log n)$ computation.

\paragraph{Backward Computation.}
In the backward pass, we compute the adjoint projection $\mathbf{\Phi}^\top \bm{v}$, where $\bm{v} \in \mathbb{R}^m$ denotes a low-dimensional one-bit vector.

Specifically, we first lift $\bm{v} \in \mathbb{R}^{m} $ to $\tilde{\bm{v}} \in \mathbb{R}^{n'}$ by zero-padding as follows
\begin{align}
\tilde{\bm{v}} = \bm{S}'^\top \bm{v}.
\end{align}

This operation places the elements of $\bm{v}$ into the coordinates that were selected by $\bm{S}$, filling the rest with zeros.
We then apply the FHT to the padded vector, and apply the same random sign flips as in the forward pass as $\bm{D} \bm{H}\tilde{\bm{v}}$.

Finally, we truncate the resulting $n'$-dimensional vector back to the original dimensions. 
This is the action of $\bm{P}_{\text{pad}}^\top$.
The operator $\bm{P}_{\text{pad}}^\top:\mathbb{R}^{n'}\rightarrow \mathbb{R}^n$ is a truncation operator, which we denote as $\bm{P}_{\text{trunc}}$.
Its action is to select the first $n$ coordinates of a vector in $\mathbb{R}^{n'}$ and discard the remaining $n'-n$ coordinates.

The complete operation is formulated as
\begin{equation}
\begin{aligned}
\mathbf{\Phi}^\top \bm{v}&=(\bm{S}' \bm{H} \bm{D} \bm{P}_{\text{pad}})^\top \bm{v} = \bm{P}_{\text{pad}}^\top \bm{D}^\top \bm{H}^\top \bm{S}'^\top \bm{v}.
\end{aligned}
\end{equation}
Given that $\bm{P}_{\text{trunc}}=\bm{P}_{\text{pad}}^\top$ and
$\bm{D}$ is diagonal, this simplifies to $\mathbf{\Phi}^\top \bm{v}=\bm{P}_{\text{trunc}} \bm{D} \bm{H}^\top \bm{S}'^\top \bm{v}$.

Similar to the forward projection, the adjoint projection is matrix-free and has a computational complexity of $\mathcal{O}(n \log n)$, making our algorithm practical and efficient.

\section{Theoretical Analysis}
We provide a rigorous convergence guarantee for pFed1BS, showing that the algorithm converges to a neighborhood of a stationary point.

\subsection{Assumptions and Preliminaries}
Our analysis relies on a set of standard assumptions common in FL literature.
\begin{assumption}[$L$-smoothness]
	\label{assump:L-smooth}
	The local objective function $f_k(\cdot)$ is $L$-smooth, \textit{i.e.}, its gradient is $L$-Lipschitz continuous.
\end{assumption}

\begin{assumption}[Bounded Below] 
	\label{assump: Bounded Below}
	The global potential function 
	$\Psi(\bm{w}_1,\ldots,\bm{w}_K;\bm{v})=\sum_{k=1}^{K} p_k \tilde{F}_k(\bm{w}_k ;\bm{v})$ is bounded below by some value $F^*$.
\end{assumption}

\begin{assumption}[Bounded Gradient Variance]
	\label{assump:bounded_variance_main}
	The variance of the stochastic gradients is uniformly bounded, i.e., for any client $k$ and model $\bm{w}$, $\mathbb{E}_{\mathcal{B}} \bigg[\|\nabla \hat{f}_k(\bm{w};\mathcal{B}) - \nabla f_k(\bm{w})\|^2 \bigg] \le \sigma^2$ for some constant $\sigma^2 \geq 0$.
\end{assumption}

\begin{assumption}[Bounded Task Gradient Variance]
	\label{assump: BTGV}
	The stochastic gradient of the task loss $f_k$ has a bounded second moment, \textit{i.e.}, there exists a constant  $G>0$ such that $\mathbb{E}_{\mathcal{B}} \bigg[\|\nabla \hat{f}_k(\bm{w};\mathcal{B}) \|^2 \bigg] \leq G^2$ for all $\bm{w}$ and $k$.
\end{assumption}

Our analysis of partial client participation also relies on the following standard lemmas.

\begin{lemma}[Bounded Projection Norm]
	\label{assump: BPN}
	Let the projection matrix $\mathbf{\Phi} \in \mathcal{R}^{m\times n}$ be constructed as described in the "Efficient Projection" section, involving a normalized Hadamard matrix $\bm{H}$, a random sign matrix $\bm{D}$, a subsampling matrix $\bm{S}$ and a padding matrix $\bm{P}_{\text{pad}}$. The resulting operator has an exact spectral norm given by:
	\begin{equation}
	\|\mathbf{\Phi}\| \le \sqrt{\frac{n'}{m}}.
	\end{equation}
\end{lemma}
For our analysis, we formally define $ C_{\Phi}  = \sqrt{\frac{n'}{m}} $, where $C_{\Phi}=\mathcal{O}\left(\sqrt{\frac{n}{m}}\right)$.

\begin{lemma}[Client-Side Objective and Gradient]
	\label{lemma: Client Gradient}
	The client-side objective $\tilde{F}_{k}$ for a given client $k$ and server message $\bm{v}$ is defined as:
	\begin{align}
	\label{eq:obj}
	\tilde{F}_{k}(\bm{w}_k; \bm{v}) &= f_k(\bm{w}_k) + \lambda \left (h_\gamma(\mathbf{\Phi}\bm{w}_k) - \langle \bm{v}, \mathbf{\Phi}\bm{w}_k \rangle \right) \nonumber \\
	&+\frac{\mu}{2} \|\bm{w}_k\|_2^2,
	\end{align}
	\noindent where $h_\gamma(\bm{z})=\frac{1}{\gamma}\sum_{i=1}^{m}\log(\cosh(\gamma z_i))$ is a differentiable surrogate for the $\ell_1$-norm. 
	The objective $\tilde{F}_k(\bm{w}_k)$ is differentiable with respect to $\bm{w}_k$, and its gradient is given by:
	\begin{align}
	\label{eq:client_gradient}
	\nabla \tilde{F}_k(\bm{w}_{k}; \bm{v}) = &
	\nabla f_k(\bm{w}_{k}) + \lambda \mathbf{\Phi}^{\top} \left[\tanh(\gamma \mathbf{\Phi} \bm{w}_{k}) - \bm{v} \right] \nonumber \\
	&+\mu \bm{w}_{k}.
	\end{align}
\end{lemma}

\begin{lemma}[Smoothness of Client Objective] 
	\label{lm:LF_smooth}
	Under Assumptions~\ref{assump:L-smooth} and~\ref{assump: BPN}, the client-side objective $\tilde{F}_{k}(\bm{w}_k;\bm{v})$ is $L_F$-smooth with respect to $\bm{w}_k$, where the smoothness constant is given by:
	\begin{equation}
	L_F = L + \lambda \gamma C_{\Phi}^2 + \mu.
	\end{equation}
\end{lemma}

\begin{lemma}[Bounded Model Norm] 
	\label{lemma:Bounded Model Norm}
	Let Assumption~\ref{assump:L-smooth}-\ref{assump: BTGV} hold.
	For a learning rate $\eta$  satisfying $\eta<\frac{1}{6\mu}$,  the expected squared norm of the client weights is uniformly bounded across all rounds $t$:
	\begin{equation}
		\mathbb{E}[\|\bm{w}_k^t\|_2^2] \leq W^2, \quad \forall t \geq 0,
	\end{equation}
	\noindent where the bound $W^2$ is defined as
	\begin{equation}
		W^2 \triangleq \|\bm{w}_k^0\|_2^2 + \frac{C'}{(1-\alpha)(1-\alpha^R)}
	\end{equation}
	with constants $\alpha=1-\eta\mu(1/2-3\eta\mu)$ and $C'$ given by:
	\begin{equation}
		C' \triangleq \left(\frac{2\eta}{\mu} + 3\eta^2\right)G^2 + 4\Big(\frac{\eta}{\mu} + 3\eta^2 \Big) \lambda^2 C_\Phi ^2 m.
	\end{equation}
\end{lemma}

\begin{lemma}[Variance of Client Sampling]
	\label{lem:sampling_variance_main} 
	Let $\{\bm{z}_k^t\}_{k=1}^{K}$ be the set of client sketches at round $t$.
	If a subset $\mathcal{S}^t$ of size $S$ is sampled uniformly at random without replacement, then the variance of the sample mean is bounded by:
	\begin{multline}
	\label{eq:sampling_variance_main}
	\mathbb{E}_{\mathcal{S}^t}\left[ \bigg\| \frac{1}{S}\sum_{k \in \mathcal{S}^t} p_k\bm{z}_k^t - \bar{\bm{z}}^t \bigg\|^2_2 \right] 
	\\ \le \frac{K-S}{SK(K-1)} \sum_{k=1}^K \left\| p_k \bm{z}_k^t - \bar{\bm{z}}^t \right\|^2_2,
	\end{multline}   
	where $\bar{\bm{z}}^t \triangleq\frac{1}{K}\sum_{k=1}^K p_k \bm{z}_k^t$.
\end{lemma}

\begin{lemma}[Client-Side Objective Descent] 
	\label{lemma:client_update}
	After $R$ local steps of subgradient descent with learning rate $\eta$ on the smoothed objective $\tilde{F}_{k}(\cdot;\bm{v}^t)$, starting from $\bm{w}_{k,0}^{t+1}=\bm{w}_{k}^{t}$, we have
	\begin{align}
	&\mathbb{E}\left[\tilde{F}_{k}(\bm{w}_{k,R}^{t+1}; \bm{v}^t)\right] 
	\le \tilde{F}_{k}(\bm{w}_k^t; \bm{v}^t) + \frac{\eta^2 R L_F \sigma^2}{2}  \nonumber \\
	&\quad - \eta R \left(1 - \frac{\eta L_F}{2}\right) 
	\cdot \frac{1}{R} \sum_{r=0}^{R-1} 
	\left\| \nabla \tilde{F}_{k}(\bm{w}_{k,r}^{t+1}; \bm{v}^t) \right\|_2^2.
	\end{align}
\end{lemma}

\subsection{Main Convergence Result}
Our proof relies on several key intermediate results (proofs in the Appendix). 
These are combined to analyze the evolution of a global potential function:
\begin{equation} \label{eq:potential}
\Psi^t \triangleq \sum_{k=1}^{K} p_k \tilde{F}_{k}(\bm{w}_{k}^{t}; \bm{v}^t).
\end{equation}

 Our main result bounds the time-averaged expected squared norm of the local gradients.
 
 \begin{theorem}[Local Convergence]
 	\label{thm:main_convergence}
 	Under standard assumptions for federated learning analysis (detailed in the Appendix),
 	if the learning rate satisfies $\eta \leq \frac{1}{L_F}$, after $T$ communication rounds with partial client participation, Algorithm \ref{alg:main} guarantees the following convergence
 	\begin{multline}
 	\frac{1}{T}\sum_{t=0}^{T-1}\frac{1}{R}\sum_{r=0}^{R-1} \mathbb{E}\left[\sum_{k=1}^K p_k \|\nabla \tilde{F}_{k}(\bm{w}_{k,r}^{t+1}; \bm{v}^t)\|_2^2 \right] \\ \le \frac{\Psi^0 - F^*}{c_1 T} + \frac{\eta^2 R L_F \sigma^2}{2c_1} + \frac{\Delta_{\max}}{c_1}+\frac{\lambda E_S}{c_1},
 	\end{multline}
 	where $c_1=\eta R(1-\eta L_F/2)$.
 	The term $\Psi^0$ is the initial value of the potential function in \eqref{eq:potential}, $F^*$ is its lower bound, $L_F = L + \lambda \gamma C_{\Phi}^2 + \mu$ is the smoothness constant of the client objective, 
  $\Delta_{\max}=2\lambda(\sqrt{m} W C_{\Phi}+ m)$ bounds the error from the one-bit server update, and $E_S$ bounds the error from client sampling, defined as
 	\begin{equation}
 	E_S =  2m \sqrt{\frac{K(K-S)}{S(K-1)} \sum_{k=1}^K p_k^2}.
 	\end{equation}
 \end{theorem}

\begin{remark}
	Theorem \ref{thm:main_convergence} shows that as $T$ grows, the average squared gradient norm converges to a neighborhood of zero at an $\mathcal{O}(1/(RT))$ rate, which means that the algorithm converges to a stationary point of the
	global potential function. The neighborhood (convergence error) of pFed1BS is governed by stochastic noise $\mathcal{O}(\eta L_F \sigma^2)$, communication error $\mathcal{O}(\Delta_{\max}/({\eta R }))$ and client sampling error $\mathcal{O}({\lambda E_S}/({\eta R}))$. To bound the convergence error, the regularization parameter $\lambda$ must satisfy $\lambda=\mathcal{O}\left(1/n \right)$, which simultaneously controls $L_F$, $\Delta_{\max}$ and $E_S$.
\end{remark}

\begin{remark}
	Note that the client sampling error $E_S$ vanishes when $S=K$, \textit{i.e.}, full client participation. In this case, our convergence bound recovers the result for the full participation setting.
\end{remark}

\section{Experiments}

\begin{table*}[t]
	\centering
	\renewcommand{\arraystretch}{1.3}
	\resizebox{\textwidth}{!}{%
		\begin{tabular}{l cc | cc | cc | cc | cc} 
			\toprule[1.2pt]
			\multirow{2}{*}{\textbf{Method}} & \multicolumn{2}{c|}{\textbf{MNIST}} & \multicolumn{2}{c|}{\textbf{FMNIST}} & \multicolumn{2}{c|}{\textbf{CIFAR-10}} & \multicolumn{2}{c|}{\textbf{CIFAR-100}} & \multicolumn{2}{c}{\textbf{SVHN}} \\
			\cmidrule(lr){2-3} \cmidrule(lr){4-5} \cmidrule(lr){6-7} \cmidrule(lr){8-9} \cmidrule(l){10-11}
			& {\bf{{Acc. (\%)}}} & {\bf{{Cost (MB)}}} & {\bf{{Acc. (\%)}}} & {\bf{{Cost (MB)}}} & {\bf{{Acc. (\%)}}} & {\bf{{Cost (MB)}}} & {\bf{{Acc. (\%)}}} & {\bf{{Cost (MB)}}} & {\bf{{Acc. (\%)}}} & {\bf{{Cost (MB)}}} \\
			\midrule[1.2pt]
			FedAvg   & 97.21 $\pm$ 0.48 & 31.06 & 84.40 $\pm$ 0.09 & 31.06 & 87.78 $\pm$ 1.45 & 42.85 & 59.60 $\pm$ 0.66 & 2335.85 & 96.33 $\pm$ 0.28 & 42.85 \\
			\midrule[0.8pt]
			OBDA     & 92.54 $\pm$ 0.32 & 0.97\textsubscript{\(\downarrow\)96.88\%} & 78.51 $\pm$ 0.62 & 0.97\textsubscript{\(\downarrow\)96.88\%} & 73.26 $\pm$ 5.39 & 1.34\textsubscript{\(\downarrow\)96.88\%} & 42.47 $\pm$ 2.02 & 72.95\textsubscript{\(\downarrow\)96.88\%} & 84.32 $\pm$ 1.15 & 1.34\textsubscript{\(\downarrow\)96.88\%} \\
			OBCSAA   & 92.20 $\pm$ 0.20 & 15.58\textsubscript{\(\downarrow\)49.84\%} & 80.13 $\pm$ 0.45 & 15.58\textsubscript{\(\downarrow\)49.84\%} & 83.57 $\pm$ 0.14 & 21.49\textsubscript{\(\downarrow\)49.84\%} & 48.99 $\pm$ 0.54 & 1171.57\textsubscript{\(\downarrow\)49.84\%} & 87.10 $\pm$ 0.65 & 21.49\textsubscript{\(\downarrow\)49.84\%} \\
			zSignFed & 94.83 $\pm$ 0.07 & 16.01\textsubscript{\(\downarrow\)48.45\%} & 82.55 $\pm$ 0.28 & 16.01\textsubscript{\(\downarrow\)48.45\%} & 67.60 $\pm$ 3.46 & 22.04\textsubscript{\(\downarrow\)48.56\%} & 40.17 $\pm$ 2.32 & 1203.78\textsubscript{\(\downarrow\)48.45\%} & 85.33 $\pm$ 1.05 & 22.04\textsubscript{\(\downarrow\)48.56\%} \\
			EDEN     & 96.50 $\pm$ 0.35 & 12.15\textsubscript{\(\downarrow\)60.88\%} & 83.85 $\pm$ 0.30 & 12.15\textsubscript{\(\downarrow\)60.88\%} & 84.91 $\pm$ 0.53 & 22.76\textsubscript{\(\downarrow\)46.88\%} & 47.55 $\pm$ 1.12 & 1205.33\textsubscript{\(\downarrow\)48.39\%} & 89.01 $\pm$ 0.48 & 22.76\textsubscript{\(\downarrow\)46.88\%} \\
			FedBAT   & 96.42 $\pm$ 0.41 & 11.88\textsubscript{\(\downarrow\)61.75\%} & 83.70 $\pm$ 0.35 & 11.88\textsubscript{\(\downarrow\)61.75\%} & 81.20 $\pm$ 0.95 & 22.04\textsubscript{\(\downarrow\)48.56\%} & 46.89 $\pm$ 1.25 & 1198.71\textsubscript{\(\downarrow\)48.68\%} & 88.89 $\pm$ 0.51 & 22.04\textsubscript{\(\downarrow\)48.56\%} \\
			\midrule[0.8pt]
			\textbf{pFed1BS} & 
			\textbf{97.83 $\pm$ 0.02} & 
			\textbf{0.10\textsubscript{\(\downarrow\)99.68\%}} & 
			\textbf{84.15 $\pm$ 0.21} & 
			\textbf{0.10\textsubscript{\(\downarrow\)99.68\%}} & 
			{\bf{85.21 $\pm$ 0.34}} & 
			\textbf{0.13\textsubscript{\(\downarrow\)99.69\%}} & 
			\textbf{52.88 $\pm$ 0.32} & 
			\textbf{7.30\textsubscript{\(\downarrow\)99.69\%}} &
			{\bf{95.07 $\pm$ 0.21}} &
			\textbf{0.13\textsubscript{\(\downarrow\)99.69\%}} \\
			\bottomrule[1.2pt]
		\end{tabular}
	} %
	\caption{Top-1 accuracy (\%) and one-round communication cost (MB) of FL algorithms on various datasets under a Non-IID setting. Best results in each column are highlighted in \textbf{bold}. The row for our proposed method, pFed1BS, is highlighted with a gray background for emphasis.}
	\label{tab:main_results}
\end{table*}

We empirically evaluate pFed1BS against state-of-the-art baselines on several benchmarks under challenging \textit{non-i.i.d.} conditions.

\subsection{Experimental Setup}

{\bf{Datasets and Models: }}
Our experiments use standard image classification benchmarks: MNIST, FMNIST, CIFAR-10, CIFAR-100, and SVHN. 
To further assess generalizability, we also include experiments on SVHN. 
We use a two-layer MLP for MNIST and FMNIST, and VGG architectures for the other datasets. 
We simulate a highly \textit{non-i.i.d.} environment by partitioning data among 20 clients based on labels.

{\bf{Baselines:}}
We compare pFed1BS against FedAvg and several one-bit CEFL methods: OBDA, OBCSAA, and zSignFed. 
To provide a comprehensive evaluation, we also include state-of-the-art PFL and communication-efficient baselines: EDEN \cite{vargaftik2022eden} and  FedBAT\cite{li2024fedbat}.

{\bf{Implementation Details:}}
We run experiments for 100-300 communication rounds with multiple local epochs. 
Key hyperparameters were set via a grid search to $\lambda=0.0005$, $\mu=0.00001$, and $\gamma=10000$.
The compression ratio is fixed at $m/n=0.1$.
All experiments are implemented in PyTorch and run on an NVIDIA RTX 3090 Ti GPU, with results averaged over 10 independent runs.

\subsection{Evaluation Metrics}
We use the following metrics for evaluation:
\begin{itemize}
	\item {\bf{Top-1 Accuracy:}} We report the average Top-1 test accuracy on the held-out test set, aggregated across all clients' personalized models.
	\item {\bf{Communication Cost:}} We define the per-round communication cost as the total number of bits transmitted between the server and all participating clients in a single round. 
	For pFed1BS, this is the sum of all uplink one-bit sketches (size $m$) and the downlink one-bit consensus vector (size $m$).
\end{itemize}

\subsection{Main Results}
Table \ref{tab:main_results} presents the final test accuracy and per-round communication cost. 
pFed1BS establishes a new state-of-the-art for communication-constrained FL. 
On all datasets, it achieves accuracy that is highly competitive with or superior to all baselines, including full-precision FedAvg and advanced communication-efficient methods like OBDA, while reducing communication costs by over $96\%$. 
For instance, on CIFAR-10, pFed1BS achieves $85.21\%$ accuracy with only $0.13$ MB per round, whereas OBDA requires $1.34$ MB for a similar $73.26\%$ accuracy. 

\begin{figure}[t]
	\centering
	\includegraphics[width=0.9\columnwidth]{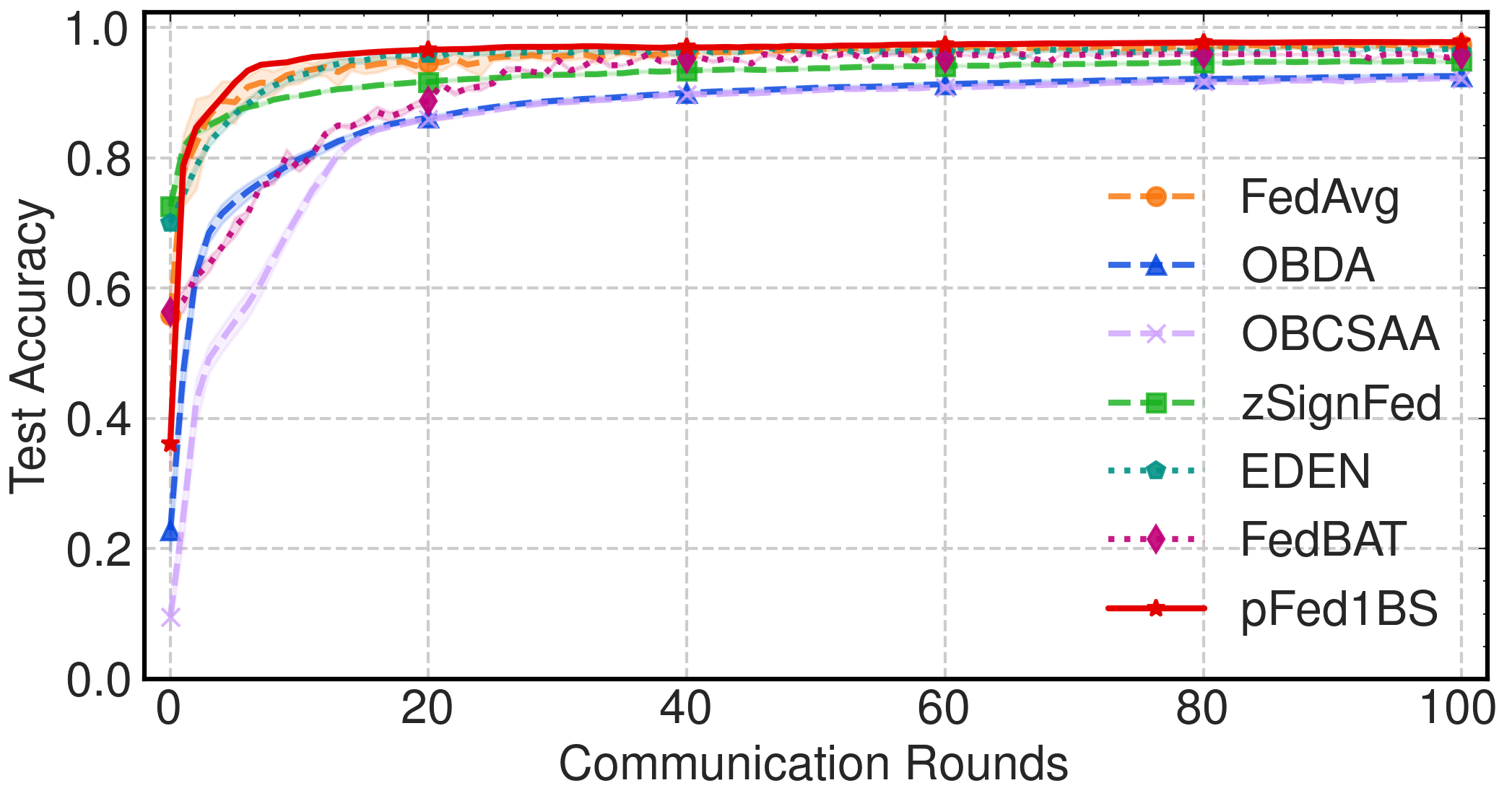}
	\caption{Test accuracy on MNIST (\textit{non-i.i.d.}). pFed1BS achieves both faster convergence and higher final accuracy.}
	\label{fig:accuracy_curve}
\end{figure}

\begin{figure}[t]
	\centering
	\includegraphics[width=0.9\columnwidth]{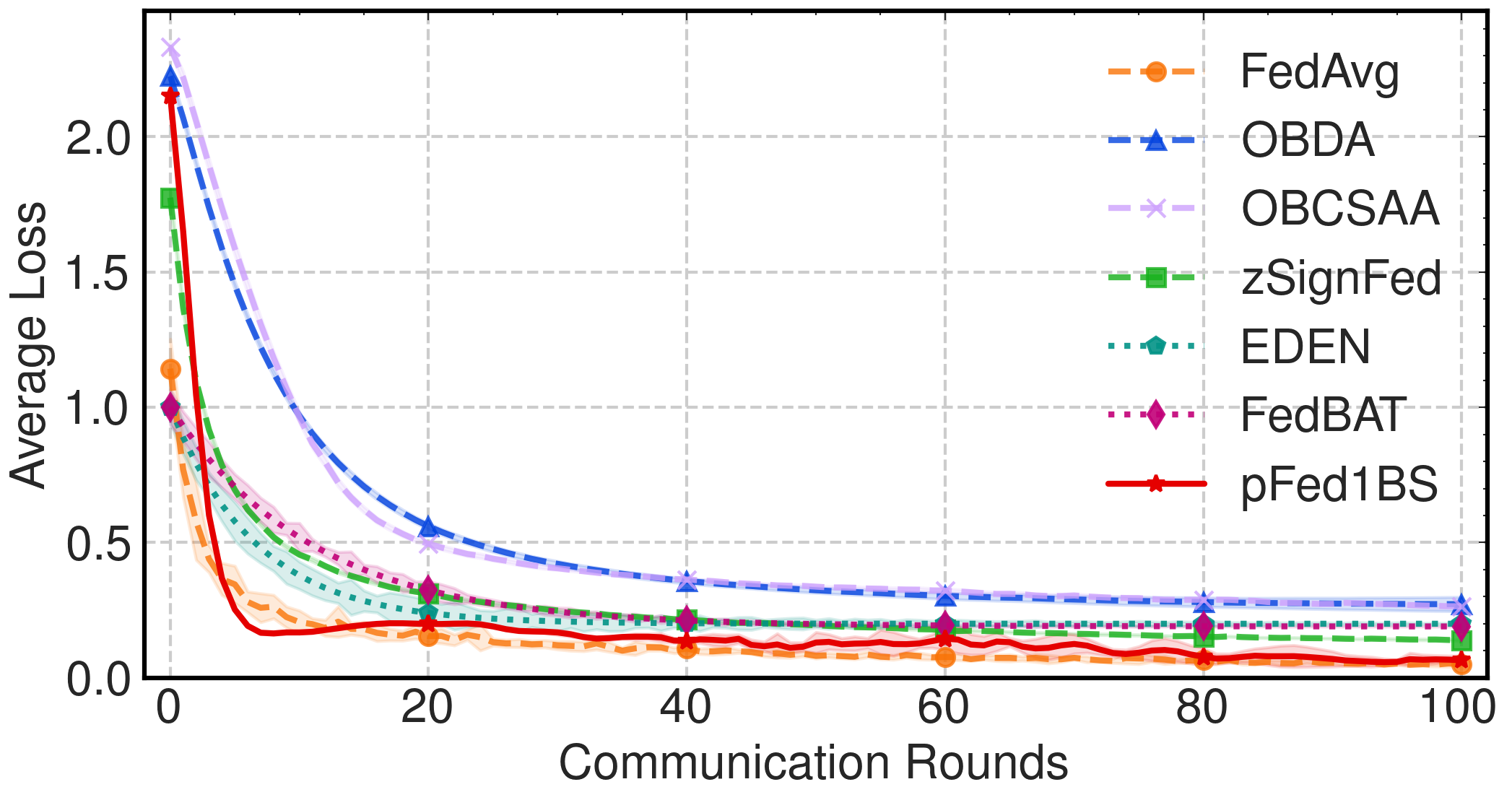}
	\caption{Average training loss on MNIST (\textit{non-i.i.d.}). 
		pFed1BS exhibits significantly faster convergence to a lower loss value and maintains stable behavior throughout the training process.}
	\label{fig:loss_curve}
\end{figure}

On more challenging datasets like CIFAR-100, the advantage is even more pronounced: one-bit baselines suffer a performance collapse, while pFed1BS maintains high accuracy, demonstrating the critical role of personalization in enabling extreme compression. 
The convergence plots in Figures \ref{fig:accuracy_curve} and \ref{fig:loss_curve} further show that pFed1BS achieves faster convergence to a better and more stable solution.

\section{Conclusion}
In this work, we proposed pFed1BS, a novel personalized federated learning framework that successfully reconciles the dual challenges of extreme communication compression and data heterogeneity. 
By integrating a bidirectional one-bit sketching mechanism with a principled sign-based regularizer, our method reduces communication costs by over $99\%$ while simultaneously achieving state-of-the-art accuracy. 
Crucially, pFed1BS achieves this extreme compression with only a minimal trade-off in model accuracy compared to full-precision methods, while decisively outperforming other one-bit baselines that suffer a significant performance collapse in \textit{non-i.i.d.} settings. 
Our work demonstrates that a carefully designed personalization strategy is the key to making extreme compression schemes viable, preventing catastrophic performance loss and establishing a new, practical frontier for deploying powerful federated models in real-world, resource-constrained environments.

\section{Acknowledgments}
This work was supported by the National Natural Science Foundation of China under Grant No. 62501432, Science and Technology on Electronic Information Control Laboratory, the Postdoctoral Fellowship Program of CPSF under Grant No. GZC20232038, and the China Postdoctoral Science Foundation under Grant No. 2024M762521.

\newpage
\appendix

\setcounter{table}{0}
\setcounter{figure}{0}
\setcounter{section}{0}
\setcounter{equation}{0}
\setcounter{lemma}{0}
\setcounter{assumption}{0}
\setcounter{theorem}{0}

\section{Appendix}
In this appendix, we provide the proofs for the theorems and lemmas in the main paper, as well as additional experimental settings and results that further validate our proposed framework, pFed1BS.

\paragraph{Appendix Contents:}
\begin{itemize}
	\item {\bf{A. Additional Experimental Results:}} We present ablation studies on the MNIST and CIFAR-10 datasets to analyze the impact of key components and hyperparameters of pFed1BS.
	\begin{itemize}
		\item A.1. Effect of the Number of Participating Clients ($S$)
		\item A.2. Effect of the Number of Local Epochs ($R$)
		\item A.3. Performance with the Fast Hadamard Transform (FHT)
		\item A.4. Hyperparameter Sensitivity Analysis
	\end{itemize}
	
	\item {\bf{B. Theoretical Analysis and Proofs:}} 
	We provide detailed proofs for all lemmas and the main convergence theorem.
	\begin{itemize}
		\item B.1. Assumptions
		\item B.2. Supporting Lemmas and Proofs
		\item B.3. Main Convergence Proof (Theorem 1)
	\end{itemize}
\end{itemize}

\subsection{A. Additional Experimental Results}
We conduct further analysis on the MNIST dataset under the non-i.i.d. setting to investigate the impact of our method's key components and hyperparameters.

\paragraph{A.1 Effect of the Number of Participating Clients ($S$).}
To understand how the degree of client participation affects performance, we vary the number of clients ($S$) sampled in each communication round, from sparse participation to full participation ($S=20$). 
As shown in Figure~\ref{fig:ablation_s}, 
model performance improves directly with the number of participating clients.
A larger $S$ provides the server with a more accurate and stable estimate of the global consensus, leading to faster convergence and higher final accuracy. 
Even with a small fraction of participating clients (e.g., $S=5$), pFed1BS maintains robust performance, demonstrating its efficacy in settings with limited client availability.


\begin{figure}[h]
	\centering
	\begin{subfigure}{\linewidth}
		\centering
		\includegraphics[width=0.8\linewidth]{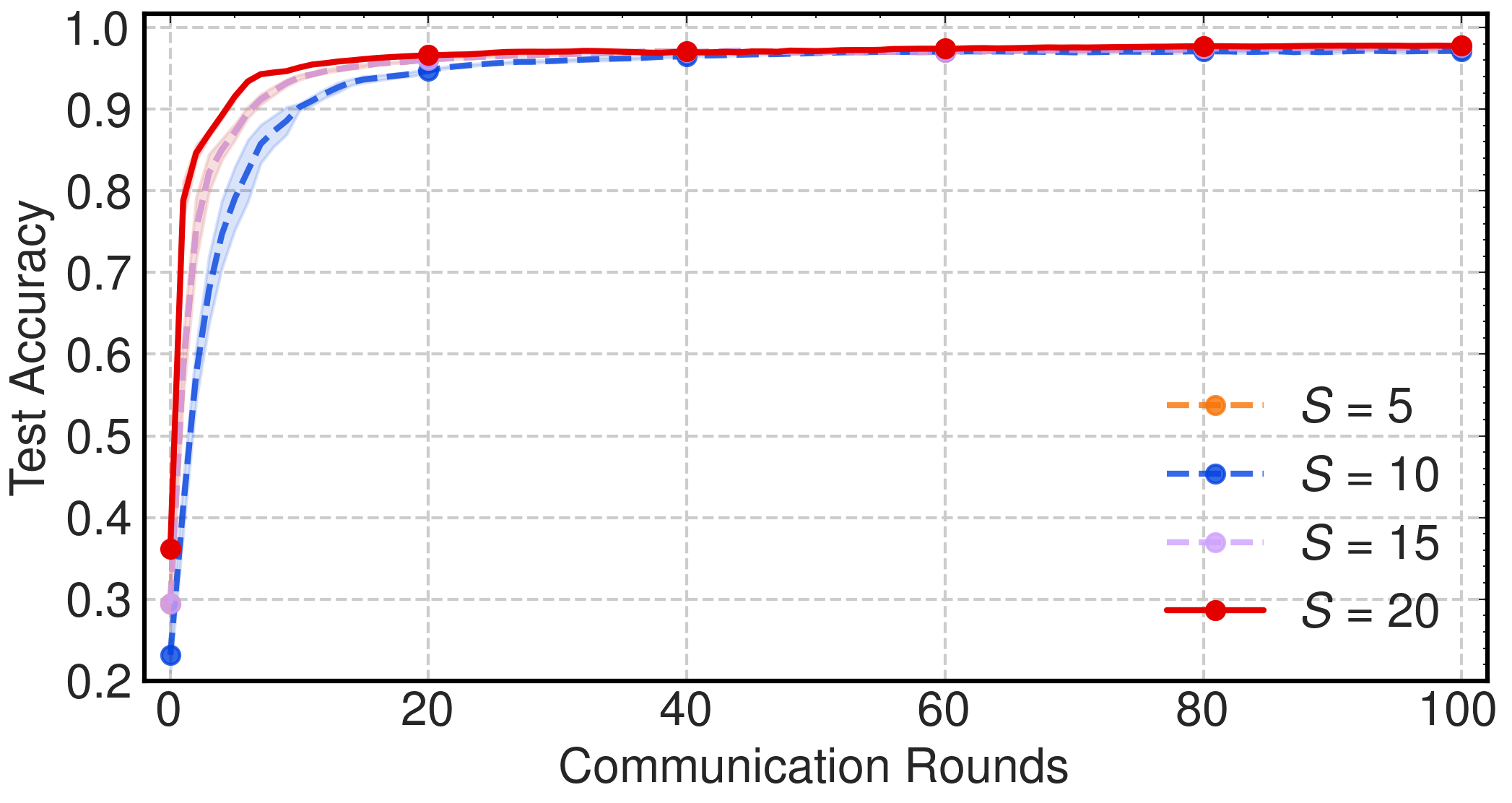}
		\subcaption{Test accuracy \textit{v.s.} communication rounds.}
		\label{fig:clients_sub_a}
	\end{subfigure}
	\vspace{1em}
	\begin{subfigure}{\linewidth}
		\centering
		\includegraphics[width=0.8\linewidth]{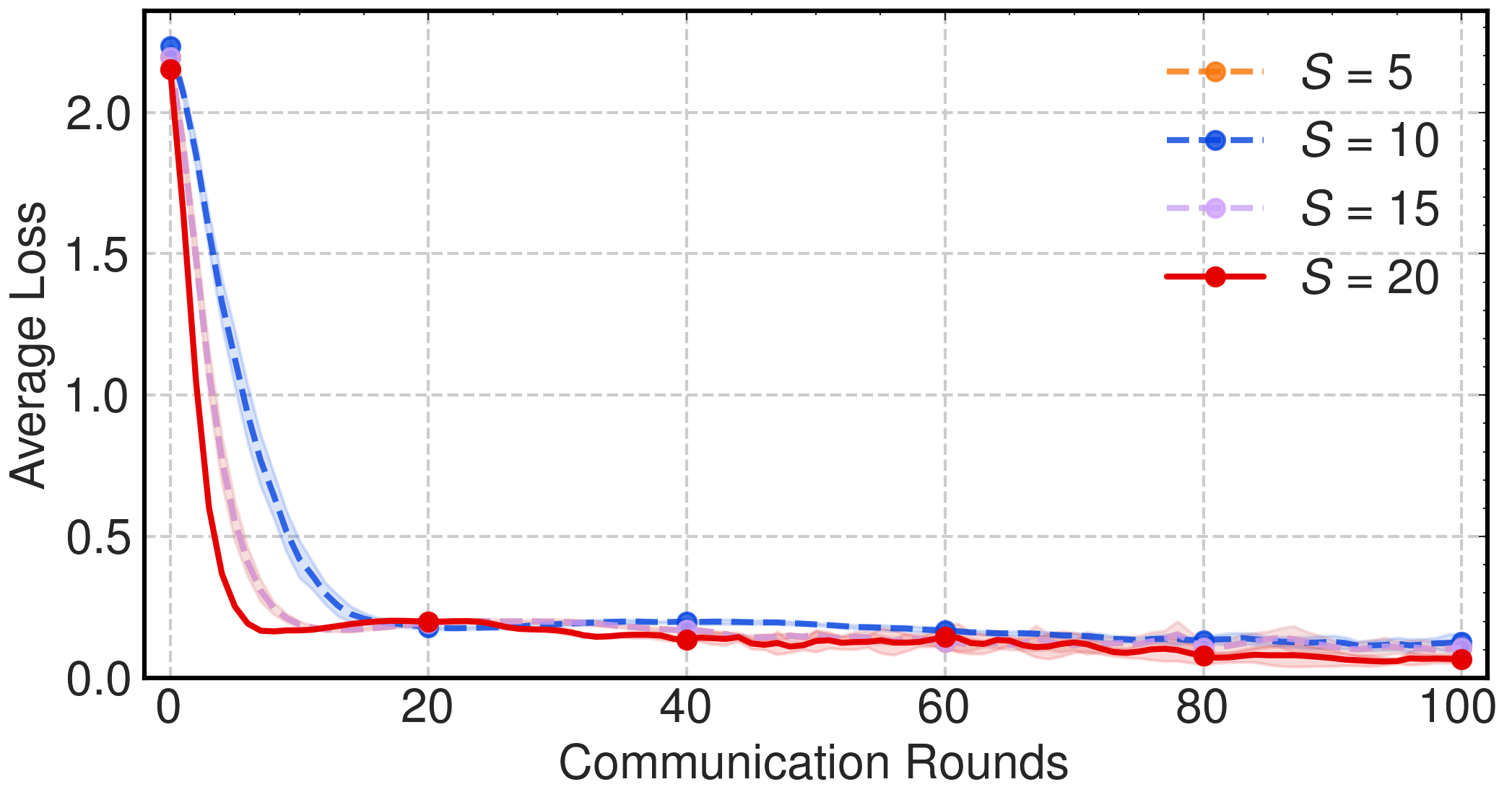}
		\subcaption{Average training loss \textit{v.s.} communication rounds.}
		\label{fig:clients_sub_b}
	\end{subfigure}
	\caption{Performance of pFed1BS with a varying number of participating clients ($S$) on MNIST.}
	\label{fig:ablation_s}
\end{figure}

\paragraph{A.2. Effect of Local Epochs ($R$).}
The number of local epochs, $R$, determines the amount of computation each client performs between communication rounds. 
We evaluate the impact of varying $R$ from $5$ to $30$.
As shown in Figure~\ref{fig:local_epochs}, increasing the amount of local work generally accelerates convergence in terms of communication rounds. 
For instance, increasing $R$ from $5$ to $20$ leads to a noticeable improvement in convergence speed. 
However, this benefit saturates quickly; performance for  $R=20$ is nearly identical.
This suggests that while sufficient local training is beneficial, excessive local updates provide diminishing returns and may not be the most efficient use of computational resources, making $R=20$ or $R=25$ a practical choice for this task.

\begin{figure}[h]
	\centering
	\begin{subfigure}{\linewidth}
		\centering
		\includegraphics[width=0.8\linewidth]{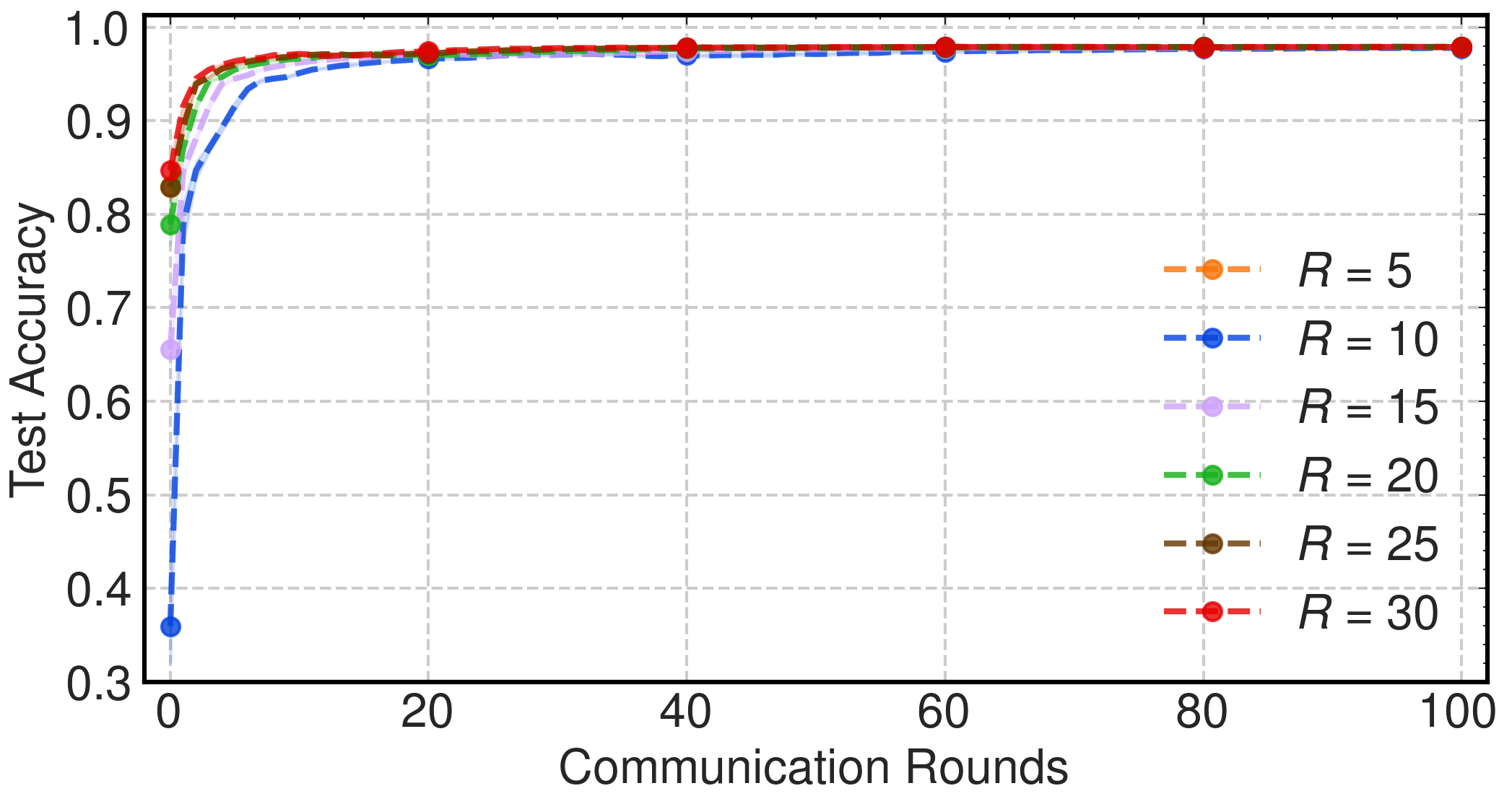}
		\subcaption{Test accuracy \textit{v.s.} communication rounds.}
		\label{fig:local_epochs_sub_a}
	\end{subfigure}
	\vspace{1em}
	\begin{subfigure}{\linewidth}
		\centering
		\includegraphics[width=0.8\linewidth]{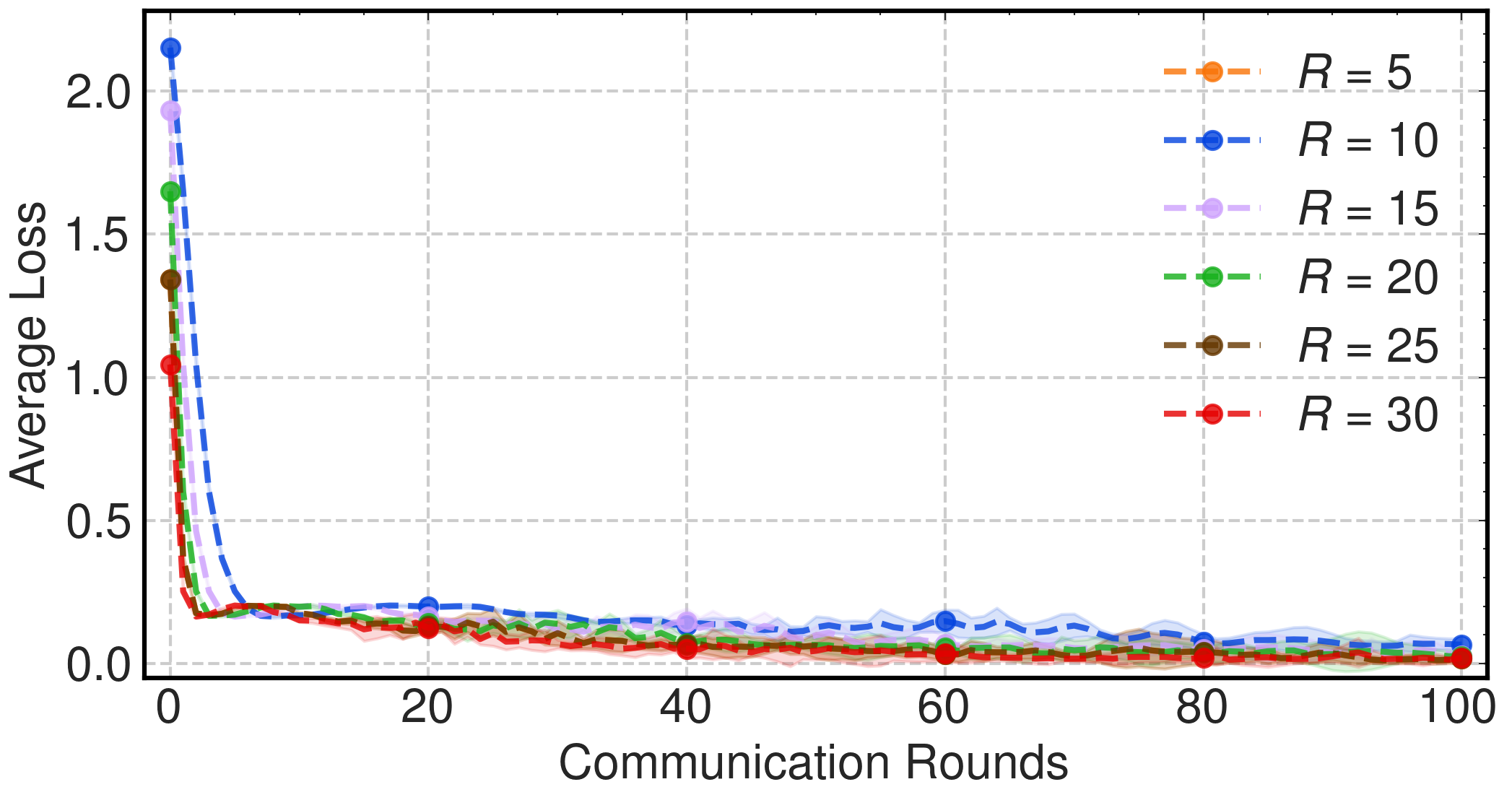}
		\subcaption{Average training loss \textit{v.s.} communication rounds.}
		\label{fig:local_epochs_sub_b}
	\end{subfigure}
	\caption{Effect of the number of local epochs ($R$) on MNIST.}
	\label{fig:local_epochs}
\end{figure}

\paragraph{A.3. Performance with the Fast Hadamard Transform (FHT).}
A key claim of our work is that the computational efficiency gained by using a structured projection (FHT) does not come at the cost of model performance. 
To validate this, we compared our FHT-based implementation against a baseline using a dense Gaussian projection matrix. 
The results in Figure~\ref{fig:ablation_fht} are decisive. The accuracy and loss curves for both methods are nearly identical throughout the training process.
This empirically confirms that our use of FHT provides its significant computational advantages (from $\mathcal(O)(mn)$ to $\mathcal(O)(n\log n)$) with no discernible impact on convergence or final model quality, making it a critical component for the scalability of pFed1BS.


\begin{figure}[h]
	\centering
	\begin{subfigure}{\linewidth}
		\centering
		\includegraphics[width=0.8\linewidth]{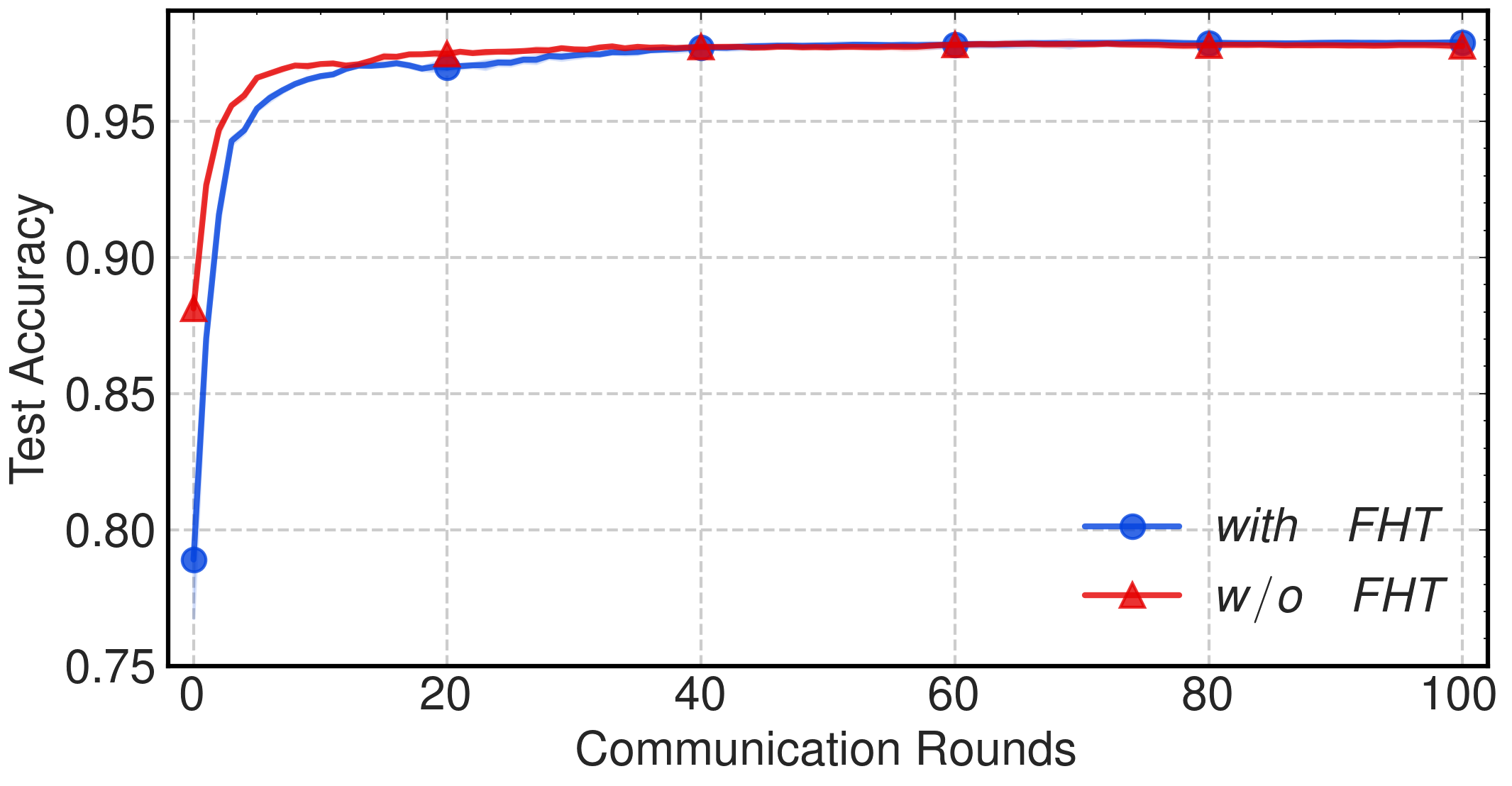}
		\subcaption{Test accuracy \textit{v.s.} communication rounds.}
		\label{fig:fht_sub_a}
	\end{subfigure}
	\vspace{1em}
	\begin{subfigure}{\linewidth}
		\centering
		\includegraphics[width=0.8\linewidth]{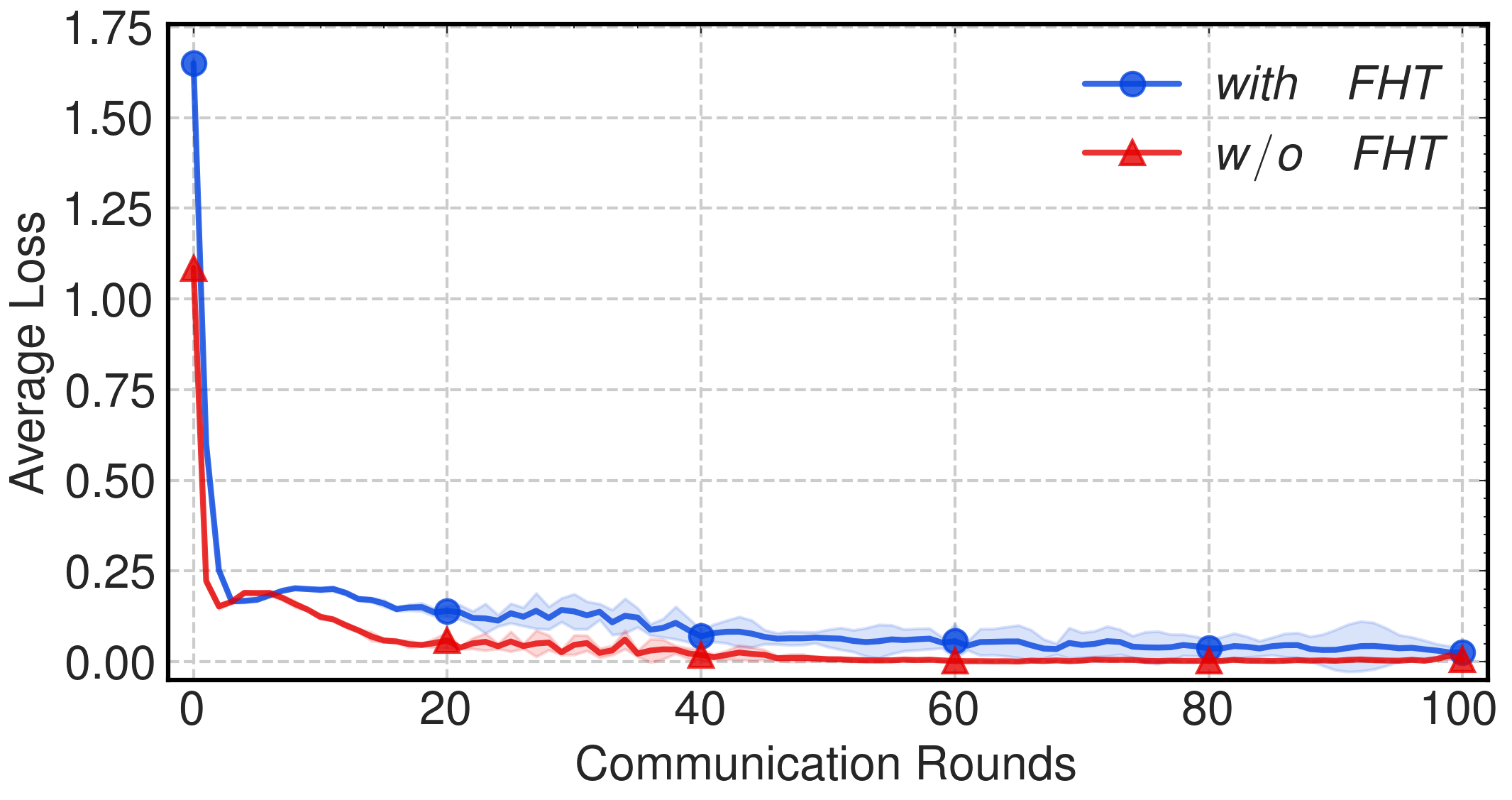}
		\subcaption{Average training loss \textit{v.s.} communication rounds.}
		\label{fig:fht_sub_b}
	\end{subfigure}
	\caption{Performance comparison of pFed1BS with a structured FHT-based projection versus a dense Gaussian projection.}
	\label{fig:ablation_fht}
\end{figure}

\paragraph{A.4. Hyperparameter Sensitivity Analysis.}
To validate the robustness of pFed1BS, we conduct a comprehensive sensitivity analysis of its key hyperparameters on the CIFAR-10 dataset under \textit{non-i.i.d.} conditions. The results are presented in Table~\ref{tab:hyperparameter_sensitivity}.

\begin{table*}[t]
	\centering
	\renewcommand{\arraystretch}{1.2}
	\begin{subtable}[t]{0.32\textwidth}
		\centering
		\label{tab:sens_lambda}
		\begin{tabular}[t]{l c}
			\toprule[1.2pt]
			\textbf{Value of $\lambda$} & \textbf{Accuracy (\%)} \\
			\midrule[0.8pt]
			$5 \times 10^{-7}$ & $86.46 \pm 0.28$ \\
			$5 \times 10^{-6}$ & $\mathbf{86.61 \pm 0.21}$ \\
			$5 \times 10^{-5}$ & $86.24 \pm 0.31$ \\
			$5 \times 10^{-4}$ & $86.28 \pm 0.25$ \\
			$5 \times 10^{-2}$ & $86.48 \pm 0.23$ \\
			$5 \times 10^{-1}$ & $86.29 \pm 0.35$ \\
			\bottomrule[1.2pt]
		\end{tabular}
		\caption{Impact of $\lambda$}
	\end{subtable}
	\hfill
	\begin{subtable}[t]{0.32\textwidth}
		\centering
		\label{tab:sens_mu}
		\begin{tabular}[t]{l c}
			\toprule[1.2pt]
			\textbf{Value of $\mu$} & \textbf{Accuracy (\%)} \\
			\midrule[0.8pt]
			$10^{-6}$ & $86.26 \pm 0.29$ \\
			$10^{-5}$ & $\mathbf{87.03 \pm 0.18}$ \\
			$10^{-4}$ & $86.10 \pm 0.24$ \\
			$10^{-3}$ & $86.34 \pm 0.22$ \\
			$10^{-2}$ & $85.96 \pm 0.33$ \\
			$10^{-1}$ & $85.99 \pm 0.41$ \\
			\bottomrule[1.2pt]
		\end{tabular}
		\caption{Impact of $\mu$}
	\end{subtable}
	\hfill
	\begin{subtable}[t]{0.32\textwidth}
		\centering
		\label{tab:sens_gamma}
		\begin{tabular}[t]{l c}
			\toprule[1.2pt]
			\textbf{Value of $\gamma$} & \textbf{Accuracy (\%)} \\
			\midrule[0.8pt]
			$10^{1}$  & $86.82 \pm 0.26$ \\
			$10^{2}$  & $86.40 \pm 0.30$ \\
			$10^{3}$  & $86.21 \pm 0.28$ \\
			$10^{4}$  & $86.56 \pm 0.24$ \\
			$10^{5}$  & $\mathbf{87.03 \pm 0.22}$ \\
			$10^{6}$  & $86.56 \pm 0.31$ \\
			\bottomrule[1.2pt]
		\end{tabular}
		\caption{Impact of $\gamma$}
	\end{subtable}
	\caption{Sensitivity analysis of the key hyperparameters of pFed1BS ($\lambda$, $\mu$, and $\gamma$) on the CIFAR-10 (Non-IID) dataset. We report the Top-1 test accuracy (\%) as 'mean ± standard deviation' over multiple independent runs. The best result in each sub-table is shown in \textbf{bold}.}
	\label{tab:hyperparameter_sensitivity} 
\end{table*}

The analysis demonstrates that pFed1BS is remarkably robust to the precise settings of its key hyperparameters. For instance, the sign-alignment parameter $\lambda$ can be varied across six orders of magnitude with a performance fluctuation of less than $0.4\%$. Similarly, the model performs consistently well for the $\ell_2$ penalty $\mu$ and the smoothing parameter $\gamma$ across very wide ranges, with performance degrading only at extreme values. This stability reduces the burden of meticulous hyperparameter tuning, making our method more practical for real-world deployment. These results confirm that our sign-based regularizer and $\ell_2$ penalty are essential components that contribute to the final performance without introducing fragility.

\subsection{B. Theoretical Analysis and Proofs}
In this subsection, we provide the detailed theoretical analysis for our proposed Algorithm 1, including all supporting lemmas and the full proof for our main convergence result, Theorem 1, which is stated in the main paper. 

\paragraph{B.1. Assumptions.}

Our analysis relies on the following assumptions. Assumption~\ref{assump: L-smoothness} and Assumption~\ref{assump: Bounded Below} are standard assumptions in the analysis of large-scale optimization and federated learning~\cite{mcmahan2017fedavg}. 
Assumptions~\ref{assump: BSGV}~\cite{karimireddy2020scaffold,yu2019linear,t2020personalized} and~\ref{assump: BTGV}~\cite{deng2020adaptive,fallah2020personalized}  are widely used in FL context.

\begin{assumption}[L-smoothness] 
	\label{assump: L-smoothness}
	The local objective function $f_k(\cdot)$ is differentiable and $L-$smooth for all clients $k \in {1,\ldots,K}$, \textit{i.e.}, for any $\bm{w}_1$, $\bm{w}_2 \in \mathbb{R}^n$, there exists a constant $L>0$ such that
	\begin{equation}
	\|\nabla f_k(\bm{w}_1) - \nabla f_k(\bm{w}_2)\| \le L \|\bm{w}_1 - \bm{w}_2\|.
	\end{equation}
\end{assumption}

\begin{assumption}[Bounded Below] 
	\label{assump: Bounded Below}
	The global potential function 
	$\Psi(\bm{w}_1,\ldots,\bm{w}_K;\bm{v})=\sum_{k=1}^{K} p_k \tilde{F}_k(\bm{w}_k)$ is bounded below by some value 
	$F^*$.
\end{assumption}


\begin{assumption}[Bounded Stochastic Gradient Variance]
	\label{assump: BSGV}
	The variance of the stochastic gradients computed on mini-batches is bounded as follows
	\begin{equation}
	\mathbb{E}_{\mathcal{B}} \bigg[\|\nabla \hat{f}_k(\bm{w};\mathcal{B}) - \nabla f_k(\bm{w})\|^2 \bigg] \le \sigma^2, \quad \forall \bm{w}.
	\end{equation}
\end{assumption}
\begin{assumption}[Bounded Task Gradient Variance]
	\label{assump: BTGV}
	The stochastic gradient of the task loss $f_k$ has a bounded second moment, \textit{i.e.}, there exists a constant  $G>0$ such that $\mathbb{E}_{\mathcal{B}} \bigg[\|\nabla \hat{f}_k(\bm{w};\mathcal{B}) \|^2 \bigg] \leq G^2$ for all $\bm{w}$ and $k$.
\end{assumption}

\paragraph{B.2. Supporting Lemmas and Proofs.}

We now establish several key lemmas concerning the properties of the projection matrix norm, the client objective, model boundedness, and server aggregation.

\begin{lemma}[Bounded Projection Norm]
	\label{assump: BPN}
	Let the projection matrix $\mathbf{\Phi} \in \mathcal{R}^{m\times n}$ be constructed as described in the "Efficient Projection" section, involving a normalized Hadamard matrix $\bm{H}$, a random sign matrix $ \bm D$, a subsampling matrix $\bm S$ and a padding matrix $\bm{P}_{\text{pad}}$. The resulting operator has an exact spectral norm given by:
	\begin{equation}
	\|\mathbf{\Phi}\| \le \sqrt{\frac{n'}{m}},
	\end{equation}
	
	For our analysis, we formally assume that $C_{\Phi} = \sqrt{\frac{n'}{m}}$, where $C_{\Phi}=\mathcal{O}\left(\sqrt{\frac{n}{m}}\right)$.
	
\end{lemma}

\begin{proof}
	The proof relies on analyzing the spectral norm of the operator $\mathbf{\Phi}$.
	Let us define the projection matrix as $\mathbf{\Phi} = \sqrt{\frac{n'}{m}}\bm{S} \bm{H} \bm{D} \bm{P}_{\text{pad}}$, where the dimensions and properties of the matrices are as described in the lemma and the "Efficient Projection" section. 
	Specifically, $\bm{D}$  is a diagonal sign matrix, satisfying $\bm{D}^{\top}\bm{D}=\bm{I}_{n'}$. 
	$\bm{H}$ is a normalized Hadamard matrix, satisfying $\bm{H}^{\top}\bm{H}=\bm{I}_{n'}$.
	$\bm{S}$ is a subsampling matrix that selects $m$ rows from $n'$, satisfying $\bm{S}\bm{S}^{\top}=\bm{I}_m$.
	Let us define an intermediate matrix $\mathbf{Q} =\bm{S} \bm{H} \bm{D}$. Then $\mathbf{\Phi} = \sqrt{\frac{n'}{m}}\mathbf{Q} \bm{P}_{\text{pad}}$. The spectral norm of  $\mathbf{\Phi}$  is given by $\|\mathbf{\Phi}\|=\sqrt{\frac{n'}{m}}\|\mathbf{Q}\|$.
	
	The core of the proof is to compute the spectral norm of $\mathbf{Q}$. We can do this by analyzing the matrix   $\mathbf{Q}\mathbf{Q}^{\top}$:
	\begin{equation}
	\begin{aligned}
	\mathbf{Q}\mathbf{Q}^T &= (\mathbf{S} \mathbf{H} \mathbf{D})(\mathbf{S} \mathbf{H} \mathbf{D})^T \\
	&= (\mathbf{S} \mathbf{H} \mathbf{D})(\mathbf{D}^T \mathbf{H}^T \mathbf{S}^T) \\
	&= \mathbf{S} \mathbf{H} (\mathbf{D}\mathbf{D}^T) \mathbf{H}^T \mathbf{S}^T \\
	&= \mathbf{S} \mathbf{H} \mathbf{I}_{n} \mathbf{H}^T \mathbf{S}^T \quad (\text{since } \mathbf{D}\mathbf{D}^T = \mathbf{I}_{n'}) \\
	&= \mathbf{S} (\mathbf{H} \mathbf{H}^T) \mathbf{S}^T \\
	&= \mathbf{S} \mathbf{I}_{n'} \mathbf{S}^T \quad (\text{since } \mathbf{H}\mathbf{H}^T = \mathbf{I}_{n'}) \\
	&= \mathbf{S}\mathbf{S}^T = \mathbf{I}_m.
	\end{aligned}
	\end{equation}
	
	The eigenvalues of an identity matrix $\mathbf{I}_m$ are all $1$.
	The eigenvalues of $\mathbf{Q}\mathbf{Q}^{\top}$ are the squared singular values of $\mathbf{Q}$. Therefore, all non-zero singular values of $\mathbf{Q}$ must be $1$. The spectral norm $\|\mathbf{Q}\|$ is defined as the largest singular value, so we have:
	\begin{equation}
	\|\mathbf{Q}\| = \sigma_{\max}(\mathbf{Q}) = 1.
	\end{equation}
	
	Substituting this result back into the expression for the norm of $\mathbf{\Phi}$ , we get:
	\begin{equation}
	\|\mathbf{\Phi}\| \le \sqrt{\frac{n'}{m}} \cdot \|\mathbf{Q}\| \|\bm{P}_{\text{pad}}\| = \sqrt{\frac{n'}{m}} \cdot 1 = \sqrt{\frac{n'}{m}},
	\end{equation}
	where we use the fact that $\|\bm{P}_{\text{pad}}\|=1$.
	
	Finally, since $n'$ is the smallest power of two greater than or equal to $n$, we can express the norm in asymptotic notation for our analysis:
	\begin{equation}
	C_{\Phi} = \mathcal{O}\left(\sqrt{\frac{n}{m}}\right).
	\end{equation}
	This completes the proof.
\end{proof}

\begin{lemma}[Client-Side Objective and Gradient]
	\label{lemma: Client Gradient}
	The client-side objective $\tilde{F}_{k}$ for a given client $k$  and server message $\bm{v}$ is defined as:
	\begin{align}
	\label{eq:obj}
	\tilde{F}_{k}(\bm{w}_k; \bm{v}) &= f_k(\bm{w}_k) + \lambda \left (h_\gamma(\mathbf{\Phi}\bm{w}_k) - \langle \bm{v}, \mathbf{\Phi}\bm{w}_k \rangle \right) \nonumber \\
	&+\frac{\mu}{2} \|\bm{w}_k\|_2^2,
	\end{align}
	\noindent where $h_\gamma(\bm{z})=\frac{1}{\gamma}\sum_{i=1}^{m}\log(\cosh(\gamma z_i))$  is a differentiable surrogate for the $\ell_1$-norm. 
	The objective $\tilde{F}_k(\bm{w}_k)$ is differentiable with respect to $\bm{w}_k$, and its gradient is given by:
	\begin{align}
	\label{eq:client_gradient}
	\nabla \tilde{F}_k(\bm{w}_{k}; \bm{v}) &=
	\nabla f_k(\bm{w}_{k}) + \lambda \mathbf{\Phi}^{\top} \left[\tanh(\gamma \mathbf{\Phi} \bm{w}_{k}) - \bm{v} \right] \nonumber \\
	&+\mu \bm{w}_{k}.
	\end{align}
\end{lemma}

\begin{proof}
	The proof follows by direct differentiation of Equation~\ref{eq:obj}. 
	The gradient of the standard loss $f_k(\bm{w}_{k})$ and the $\ell_2$ regularization term $\frac{\mu}{2} \|\bm{w}_k\|_2^2$ are $\nabla f_k(\bm{w}_{k})$ and $\mu \bm{w}_{k}$, respectively. 
	The key is the gradient of the term $\lambda \left (h_\gamma(\mathbf{\Phi}\bm{w}_k) - \langle \bm{v}, \mathbf{\Phi}\bm{w}_k \rangle \right)$.
	Using the chain rule, the gradient of $h_\gamma(\mathbf{\Phi}\bm{w}_k)$ with respect to $\bm{w}_k$ is
	\begin{align}
	\nabla_{\bm{w}_k} h_\gamma(\mathbf{\Phi}\bm{w}_k) &= \mathbf{\Phi}^{\top} \nabla_{\bm{z}} h_\gamma(\bm{z}) \Big |_{\bm{z}=\mathbf{\Phi}\bm{w}_k} \nonumber \\ 
	&= \mathbf{\Phi}^{\top} \tanh(\gamma \mathbf{\Phi} \bm{w}_k).
	\end{align}
	
	The gradient of $-\lambda \langle \bm{v}, \mathbf{\Phi}\bm{w}_k \rangle$ is $-\lambda \mathbf{\Phi}^{\top}\bm{v}$. 
	Combining all terms yields the expression in Equation~\eqref{eq:client_gradient}.
\end{proof}

\begin{remark}
	As the smoothing parameter $\gamma \to \infty$, the function $\tanh(\gamma z_i)$ converges pointwise to $\mbox{sign}(z_i)$, and the smooth surrogate $h_\gamma(\mathbf{\Phi}\bm{w}_k)$ approaches the non-smooth $\ell_1$-norm $\|\mathbf{\Phi}\bm{w}_k\|_1$. 
	Consequently, for a large $\gamma$, the gradient descent on $\tilde{F}_k(\bm{w}_k)$ closely approximates subgradient descent on the non-smooth objective:
	\begin{equation}
	\nabla \tilde{F}_k(\bm{w}_k; \bm{v}) \approx \nabla f_k(\bm{w}_k) + \lambda \mathbf{\Phi}^{\top}({\mbox{\rm sign}}(\mathbf{\Phi}\bm{w}_k) - \bm{v}) + \mu \bm{w}_{k}.
	\end{equation}
\end{remark}

\begin{lemma}[Smoothness of Client Objective] 
	\label{lm:LF_smooth}
	Under Assumptions~\ref{assump: L-smoothness} and~\ref{assump: BPN}, the client-side objective $\tilde{F}_{k}(\bm{w}_k;\bm{v})$ is $L_F$-smooth with respect to $\bm{w}_k$, where the smoothness constant is given by:
	\begin{equation}
	L_F = L + \lambda \gamma C_\Phi^2 + \mu.
	\end{equation}
\end{lemma}

\begin{proof}
	A function is $L_F$-smooth  if its Hessian $\nabla^2 \tilde{F}_{k}(\bm{w}_k;\bm{v})$ satisfies $\|\nabla^2 \tilde{F}_{k}(\bm{w}_k;\bm{v})\|\leq L_F$ for all $\bm{w}_k$.
	We analyze the Hessian of $\tilde{F}_{k}(\bm{w}_k;\bm{v})$ by decomposing it into its three constituent parts:
	\begin{align}
	\tilde{F}_k(\bm{w}_k; \bm{v}) &= f_k(\bm{w}_k) + \lambda \left( h_{\gamma}(\mathbf{\Phi}\bm{w}_k) - \langle \bm{v}, \mathbf{\Phi}\bm{w}_k \rangle \right) \nonumber \\
	&+ \frac{\mu}{2} \|\bm{w}_k\|_2^2.
	\end{align}
	
	The Hessian of the sum is the sum of the Hessians. We bound the spectral norm of the Hessian for each part as follows.
	\begin{itemize}
		
		\item By Assumption \ref{assump: L-smoothness}, $f_k$ is $L$-smooth, which means $\|\nabla^2 f_{k}(\bm{w}_k)\|\leq L$.
		\item The $\ell_2$ regularization term $\frac{\mu}{2} \|\bm{w}_k\|_2^2$ has a gradient of $\mu \bm{w}_k$ and a constant Hessian of $\nabla^2 (\frac{\mu}{2} \|\bm{w}_k\|_2^2)=\mu I$.
		The spectral norm of this Hessian is $\|\mu I\|=\mu$.
		\item  The term $\langle \bm{v}, \mathbf{\Phi}\bm{w}_k \rangle$ is linear in $\bm{w}_k$, so its Hessian is zero. 
		\item  We only need to analyze the Hessian of $\lambda h_{\gamma(\mathbf{\Phi}\bm{w}_k)}$.
		From Lemma \ref{lemma: Client Gradient}, the gradient is $\lambda \mathbf{\Phi}^{\top} \tanh(\gamma \mathbf{\Phi} \bm{w}_k)$. 
		Taking the derivative with respect to $\bm{w}_k$ again yields the Hessian:
		\begin{equation}
		\nabla^2 (\lambda h_\gamma(\mathbf{\Phi}\bm{w})) = \lambda \gamma \mathbf{\Phi}^{\top} \mathrm{diag}\left(\mathrm{sech}^2(\gamma (\mathbf{\Phi}\bm{w}_k)_i)\right) \mathbf{\Phi}.
		\end{equation}
		The spectral norm of this Hessian is bounded as follows:
		\begin{align}
		&\|\lambda\gamma \mathbf{\Phi}^{\top} \mathrm{diag}(\mathrm{sech}^2(\gamma (\mathbf{\Phi}\bm{w}_k)_i) \mathbf{\Phi}\| \nonumber \\
		&\le \lambda\gamma \|\mathbf{\Phi}^{\top}\| \|\mathrm{diag}(\mathrm{sech}^2(\gamma (\mathbf{\Phi}\bm{w}_k)_i)\| \|\mathbf{\Phi}\| \nonumber \\
		&= \lambda\gamma \|\mathrm{diag}(\mathrm{sech}^2(\gamma (\mathbf{\Phi}\bm{w}_k)_i)\| \|\mathbf{\Phi}\|^2.
		\end{align}
		
		Since $0 \le \mathrm{sech}^2(z)\leq 1$ for any $z \in \mathbb{R}$, the spectral norm of the diagonal matrix $\mathrm{diag}(\mathrm{sech}^2(\gamma (\mathbf{\Phi}\bm{w}_k)_i)$ is $\|\mathrm{diag}(\mathrm{sech}^2(\gamma (\mathbf{\Phi}\bm{w}_k)_i)\| \leq 1$. 
		From Lemma~\ref{assump: BPN}, we have $\|\mathbf{\Phi}\| = C_\Phi$.
		Substituting these bounds gives:
		\begin{equation}
		\|\nabla^2 (\lambda h_\gamma(\mathbf{\Phi}\bm{w}_k))\| \le \lambda\gamma C_\Phi^2.
		\end{equation}
	\end{itemize}
	By the triangle inequality, the norm of the total Hessian is bounded by the sum of the norms of the individual Hessians:
	\begin{equation}
	\|\nabla^2 \tilde{F}_k(\bm{w}_k;\bm{v})\| \le L + \lambda\gamma C_\Phi^2 + \mu.
	\end{equation}
\end{proof}

\begin{lemma}[Bounded Model Norm] 
	\label{lemma:Bounded Model Norm}
	Let Assumption~\ref{assump: L-smoothness}-\ref{assump: BTGV} hold.
	For a learning rate $\eta$  satisfying $\eta<\frac{1}{6\mu}$,  the expected squared norm of the client weights is uniformly bounded across all rounds $t$:
	\begin{equation}
	\mathbb{E}[\|\bm{w}_k^t\|_2^2] \leq W^2, \quad \forall t \geq 0,
	\end{equation}
	\noindent where the bound $W$ is defined as
	\begin{equation}
	W^2 \triangleq \|\bm{w}_k^0\|_2^2 + \frac{C'}{(1-\alpha)(1-\alpha^R)}
	\end{equation}
	with constants $\alpha=1-\eta\mu(1/2-3\eta\mu)$ and $C'$ given by:
	\begin{equation}
	C' \triangleq \left(\frac{2\eta}{\mu} + 3\eta^2\right)G^2 + 4\Big(\frac{\eta}{\mu} + 3\eta^2 \Big) \lambda^2 C_\Phi ^2 m.
	\end{equation}
\end{lemma}

\begin{proof}
	For simplicity, we analyze the local updates for a single client $k$ within a communication round $t$. 
	We denote $\bm{w}_r \triangleq \bm{w}_{k,r}^{t+1}$. 
	The update rule for the client model is:
	\begin{equation}
	\bm{w}_{r+1} = \bm{w}_r - \eta {\bm{d}}_r,
	\end{equation}
	\noindent where ${\bm{d}}_r= \nabla \hat{f}_k(\bm{w}_r;\mathcal{B}_r)+\lambda\nabla \tilde{g}(\bm{v}^t,\mathbf{\Phi}\bm{w}_r)+\mu \bm{w}_r$ is the stochastic subgradient of the client objective.
	
	We analyze the evolution of the squared $\ell_2$-norm of $\bm{w}_{r+1}$:
	\begin{align}
	\label{eq:squared norm}	
	\mathbb{E}[\|\bm{w}_{r+1}\|_2^2] = \mathbb{E}[\|\bm{w}_r\|_2^2] - 2\eta\mathbb{E}[\langle \bm{w}_r, \bm{d}_r \rangle] + \eta^2\mathbb{E}[\|\bm{d}_r\|_2^2].
	\end{align}
	
	First, we bound the inner product term. 
	The inner product term $\langle \bm{w}_r, {\bm{d}}_r \rangle$ could be written as:
	\begin{align}
	\langle \bm{w}_r, {\bm{d}}_r \rangle 
	&= \langle \bm{w}_r, \nabla \hat{f}_k(\bm{w}_r;\mathcal{B}_r) + \lambda \nabla \tilde{g}(\bm{v}^t, \mathbf{\Phi}\bm{w}_r) + \mu\bm{w}_r \rangle \nonumber \\
	&= \langle \bm{w}_r, \nabla \hat{f}_k(\bm{w}_r;\mathcal{B}_r) \rangle + \lambda \langle \bm{w}_r, \nabla \tilde{g}(\bm{v}^t, \mathbf{\Phi}\bm{w}_r) \rangle \nonumber \\
	&+ \mu \|\bm{w}_r\|_2^2. \label{eq:inner_product_expansion}
	\end{align}    
	
	The term $\mu \|\bm{w}_r\|_2^2$ provides the key dissipative effect. Applying Cauchy-Schwarz and Young's inequality with a parameter $\epsilon>0$ to the first term:
	\begin{align}
		&-\mathbb{E}[\langle \bm{w}_r, \nabla \hat{f}_k(\bm{w}_r;\mathcal{B}_r) \rangle] \nonumber \\
		&\le \mathbb{E}[\|\bm{w}_r\|_2 \|\nabla \hat{f}_k(\bm{w}_r;\mathcal{B}_r)\|_2] \nonumber \\
		&\le \frac{\epsilon}{2}\mathbb{E}[\|\bm{w}_r\|_2^2] + \frac{1}{2\epsilon}\mathbb{E}[\|\nabla \hat{f}_k(\bm{w}_r;\mathcal{B}_r)\|_2^2] \nonumber  \\
		&\le \frac{\epsilon}{2}\mathbb{E}[\|\bm{w}_r\|_2^2] + \frac{G^2}{2\epsilon},
		\end{align}
	\noindent where we use the Assumption~\ref{assump: BTGV} ($\mathbb{E}[\|\nabla \hat{f}_k(\bm{w}_r;\mathcal{B}_r)\|_2^2\leq G^2$). 
	
	For the second term, using
	$\nabla \tilde{g}(\bm{v}^t,\mathbf{\Phi}\bm{w}_r)
	=
	\mathbf{\Phi}^\top(\tanh(\gamma \mathbf{\Phi}\bm{w}_r)-\bm{v}^t)$, we apply Cauchy--Schwarz again
	\begin{align}
		|\langle \bm{w}_r,\nabla \tilde{g}\rangle|
		\le
		\|\bm{w}_r\|_2\,\|\nabla \tilde{g}\|_2.
	\end{align}
	Since $\tanh(\cdot)\in[-1,1]$ element-wise and $\bm{v}^t\in\{-1,1\}^m$,
	it follows that
	\begin{align}
	\|\tanh(\gamma \mathbf{\Phi}\bm{w}_r)-\bm{v}^t\|_2 \le 2\sqrt{m}.
	\end{align}
	Therefore,
	\begin{align}
	\|\nabla \tilde{g}\|_2
	\le
	\|\mathbf{\Phi}\|_2 \cdot 2\sqrt{m}
	\le
	2C_\Phi\sqrt{m} \triangleq C_g.
	\end{align}
	Applying Young's inequality with parameter $\mu$ yields
	\begin{align}
		-2\lambda\mathbb{E}[\langle \bm{w}_r,\nabla \tilde{g}\rangle]
		\le
		\mu\mathbb{E}[\|\bm{w}_r\|_2^2]
		+
		\frac{\lambda^2}{\mu}C_g^2.
	\end{align}
	
	Combining the above bounds and choosing $\epsilon=\mu/2$, we obtain
	\begin{multline}
		-2\eta\mathbb{E}[\langle \bm{w}_r,\bm{d}_r\rangle]
		 \le
		-\frac{\eta\mu}{2}\mathbb{E}[\|\bm{w}_r\|_2^2]
		+
		\frac{2\eta}{\mu}G^2
		+
		\frac{\eta\lambda^2}{\mu}C_g^2.
	\end{multline}

	Next, we bound the squared norm of the stochastic gradient, $\mu \|{\bm{d}}_r\|_2^2$:
	\begin{align}
	\label{eq:grad_norm_bound}
	\mathbb{E}[\|{\bm{d}}_r\|_2^2] = & \mathbb{E}[\|\nabla \hat{f}_k(\bm{w}_r;\mathcal{B}_r) + \lambda \nabla \tilde{g} + \mu\bm{w}_r\|_2^2] \nonumber \\
	\le& 3( \mathbb{E}[\|\nabla \hat{f}_k(\bm{w}_r;\mathcal{B}_r)\|_2^2] + \lambda^2 \mathbb{E}[\|\nabla \tilde{g}\|_2^2 \nonumber \\
	&+ \mu^2 \mathbb{E}[\|\bm{w}_r\|_2^2] ) \nonumber \\
	\le& 3\left( G^2 + \lambda^2 C_g^2 + \mu^2\mathbb{E}[\|\bm{w}_r\|_2^2] \right), 
	\end{align}
	 where we use
	$\|\nabla \tilde{g}\|_2 \leq C_g$.
	
	Finally, we combine the bounds. Substituting everything back into the main expansion:
	\begin{align}
	\mathbb{E}[\|\bm{w}_{r+1}\|_2^2] &\le \Big(1-\frac{\eta\mu}{2}+3\eta^2\mu^2\Big)\mathbb{E}[\|\bm{w}_r\|_2^2] \nonumber \\
	&+ \left(\frac{2\eta}{\mu} + 3\eta^2\right)G^2 + \Big(\frac{\eta}{\mu} + 3\eta^2 \Big) \lambda^2 C_g^2 \nonumber \\
	&= \alpha \mathbb{E}[\|\bm{w}_r\|_2^2] + C',
	\end{align}
	where we define $\alpha \triangleq 1 - \eta\mu/2 + 3\eta^2\mu^2$ and $C' \triangleq \left(\frac{2\eta}{\mu} + 3\eta^2\right)G^2 + \Big(\frac{\eta}{\mu} + 3\eta^2 \Big) \lambda^2 C_g^2$. Notice
	\begin{equation}
		\alpha \triangleq 1-\eta\mu\!\left(\frac{1}{2}-3\eta\mu\right).
	\end{equation}
	
	Under the condition $\eta<\frac{1}{6\mu}$, we have $\alpha\in(0,1)$. This recursive inequality holds for each local step $r$ within any communication round $t+1$. 
	By applying this inequality recursively $R$ times from $r=0$ to $r=R-1$, starting with the model from the previous round $\bm{w}_k^t$:
	\begin{align}
	\mathbb{E}[\|\bm{w}_{k,R}^{t+1}\|_2^2] &\le \alpha \mathbb{E}[\|\bm{w}_{k,R-1}^{t+1}\|_2^2] + C' \nonumber \\
	&\le \alpha^R \mathbb{E}[\|\bm{w}_{k,0}^{t+1}\|_2^2] + C' \sum_{i=0}^{R-1} \alpha^i \nonumber \\
	&\le \alpha^R \mathbb{E}[\|\bm{w}_k^t\|_2^2] + \frac{C'}{1-\alpha}.
	\end{align}
	
	Since $\bm{w}_k^{t+1}$ is used as the starting point for the next round, we have established a recursive relationship across communication rounds:
	\begin{equation}
	\mathbb{E}[\|\bm{w}_k^{t+1}\|_2^2] \le \alpha^R \mathbb{E}[\|\bm{w}_k^t\|_2^2] + \frac{C'}{1-\alpha}.
	\end{equation}
	
	As $\alpha \in (0,1)$, the term $\alpha^R$  is also a constant contraction factor less than 1.
	This relationship ensures that the sequence of expected squared norms $\{\mathbb{E}[\|\bm{w}_k^t\|_2^2]\}_{t=0}^{T-1}$ cannot diverge. 
	
	Therefore, the sequence is uniformly bounded. A simple upper bound for all $t\geq 0$ is given by the maximum of its starting value and the fixed-point value of the recursion:
	\begin{align}
	\mathbb{E}[\|\bm{w}_k^t\|_2^2] \le  \|\bm{w}_k^0\|_2^2 + \frac{C'}{(1-\alpha)(1-\alpha^R)}.
	\end{align}
	
	Since all terms on the right-hand side are constants independent of $t$, this proves the uniform boundedness and completes the proof.
\end{proof}

\begin{lemma}[Variance of Client Sampling]
	\label{lem:sampling_variance} 
	Let $\{\bm{z}_k^t\}_{k=1}^{K}$ be the set of client sketches at round $t$.
	If a subset $\mathcal{S}^t$ of size $S$ is sampled uniformly at random without replacement, then the variance of the sample mean is bounded by:
	\begin{align}
	\label{eq:sampling_variance}	
	\mathbb{E}_{\mathcal{S}^t}\left[ \left\| \hat{\bm{z}}^t - \bar{\bm{z}}^t \right\|^2_2 \right] \le \frac{(K-S)}{SK(K-1)} \sum_{k=1}^K \left\| p_k\bm{z}_k^t - \bar{\bm{z}}^t \right\|^2_2.
\end{align}
	where $\bar{\bm{z}}^t \triangleq\frac{1}{K}\sum_{k=1}^K p_k \bm{z}_k^t$ and and $\hat{\bm{z}}^t= \frac{1}{S}\sum_{k \in \mathcal{S}^t} p_k \bm{z}_k^t $ is the unbiased estimator of $\bar{\bm{z}}^t$.
\end{lemma}

\begin{proof}
This is a standard result from sampling theory. For completeness, we provide a sketch of the proof. Let $\bm{\alpha}_k = p_k\bm{z}_k^t - \bar{\bm{z}}^t$. Note that $ \hat{\bm{z}}^t - \bar{\bm{z}}^t = \frac{1}{S}\sum_{k \in \mathcal{S}^t} \bm{\alpha}_k$ and $\sum_{k=1}^K \bm{\alpha}_k = \bm{0}$. 
The left-hand side becomes $\mathbb{E}[ \| \frac{1}{S}\sum_{k \in \mathcal{S}^t} \bm{\alpha}_k \|^2 ]$. We expand the squared norm:
\begin{align}
	&\mathbb{E}\left[ \left\| \frac{1}{S}\sum_{k \in \mathcal{S}^t} \bm{\alpha}_k \right\|_2^2 \right] \nonumber \\
	&= \frac{1}{S^2} \mathbb{E}\left[ \sum_{k \in \mathcal{S}^t} \|\bm{\alpha}_k\|_2^2 + \sum_{i,j \in \mathcal{S}^t, i \ne j} \langle \bm{\alpha}_i, \bm{\alpha}_j \rangle \right] \nonumber \\
	&= \frac{1}{S^2} \left( \sum_{k=1}^K p_{k \in \mathcal{S}^t}\|\bm{\alpha}_k\|_2^2 + \sum_{i \ne j} p_{i,j \in \mathcal{S}^t} \langle \bm{\alpha}_i, \bm{\alpha}_j \rangle \right) \nonumber \\
	&= \frac{1}{S^2} \left( \frac{S}{K}\sum_{k=1}^K \|\bm{\alpha}_k\|_2^2 + \frac{S(S-1)}{K(K-1)}\sum_{i \ne j} \langle \bm{\alpha}_i, \bm{\alpha}_j \rangle \right).
\end{align}

Using the property $\sum_{i \ne j} \langle \bm{\alpha}_i, \bm{\alpha}_j \rangle = - \sum_{k=1}^K \|\bm{\alpha}_k\|_2^2$, we substitute and simplify:
\begin{align}
	\mathbb{E}\left[ \left\| \frac{1}{S}\sum_{k \in \mathcal{S}^t} \bm{\alpha}_k \right\|_2^2 \right] &= \frac{1}{S^2} \left( \frac{S}{K} - \frac{S(S-1)}{K(K-1)} \right) \sum_{k=1}^K \|\bm{\alpha}_k\|_2^2 \nonumber \\
	&= \frac{K-S}{S K(K-1)} \sum_{k=1}^K \|\bm{\alpha}_k\|_2^2,
\end{align}
which is equivalent to the stated result.
\end{proof}

We now establish two key lemmas regarding the server and client updates before presenting the main convergence theorem.

\begin{lemma}[Optimal Server Aggregation]
	\label{lemma: server agg}
	For any set of client sign vectors $\{\bm{z}_k\}_{k \in \mathcal{S}^t} \subset \{\pm 1\}^m$, the server update rule
	\begin{equation}
	\bm{v} = \mbox{\rm sign} \left(\sum_{k \in \mathcal{S}^t} p_k \bm{z}_k \right)
	\end{equation}
	is the exact minimizer of the server's objective $\min_{\bm{v}\in \{\pm 1\}^m}\sum_{k \in \mathcal{S}^t}p_k {g}(\bm{v},\bm{z}_k)$.
\end{lemma}

\begin{proof}
	Let $\bm{z}_k=\mbox{sign}(\mathbf{\Phi}\bm{w}_k^{t+1})$. 
	Since both $\bm{v}$ and $\bm{z}_k$ are in $\{\pm 1\}^m$, we can write ${g}(\bm{v},\bm{z}_k)=\frac{1}{2} ( \|\bm{z}_k\|_1 - \langle \bm{v}, \bm{z}_k \rangle)$, where $m$ is the dimension of $\bm{z}_k$.
	The objective becomes:
	\begin{align}
	& \min_{\bm{v} \in \{\pm 1\}^m} \sum_{k \in \mathcal{S}^t} p_k {g}(\bm{v}, \bm{z}_k) \nonumber \\
	&= \min_{\bm{v} \in \{\pm 1\}^m} \sum_{k \in \mathcal{S}^t} \frac{p_k}{2} 
	\left( \|\bm{z}_k\|_1 - \langle \bm{v}, \bm{z}_k \rangle \right) \nonumber \\
	&= \min_{\bm{v} \in \{\pm 1\}^m} 
	\left( \frac{1}{2} \sum_{k \in \mathcal{S}^t} p_k \|\bm{z}_k\|_1 
	- \frac{1}{2} \sum_{k \in \mathcal{S}^t} p_k \langle \bm{v}, \bm{z}_k \rangle \right).
	\label{eq:quant-opt}
	\end{align}
	
	Since the first term is constant with respect to $\bm{v}$, and $\|\bm{z}_k\|_1=m$, the problem is equivalent to:
	\begin{equation}
	\max_{\bm{v} \in \{\pm 1\}^m} \sum_{k \in \mathcal{S}^t} p_k \langle \bm{v}, \bm{z}_k \rangle = \max_{\bm{v} \in \{\pm 1\}^m} \left\langle \bm{v}, \sum_{k \in \mathcal{S}^t} p_k \bm{z}_k \right\rangle.
	\end{equation}
	
	This is a dot product between a variable vector $\bm{v}$ and a fixed vector $\bar{\bm{z}}=\sum_{k \in \mathcal{S}^t} p_k \bm{z}_k $. 
	To maximize the dot product under the constraint that each component $v_j \in \{\pm 1\}$, we must choose $v_j$ to have the same sign as the corresponding component $\bar{z}_j$.
	Therefore, the optimal solution is:
	\begin{align}
	\bm{v} = \mbox{sign}(\bar{\bm{z}}) = \mbox{sign} \left(\sum_{k \in \mathcal{S}^t} p_k \bm{z}_k \right).
	\end{align}
\end{proof}

\begin{lemma}[Client-Side Objective Descent] 
	\label{lemma:client_update}
	After $R$ local steps of subgradient descent with learning rate $\eta$ on the smoothed objective $\tilde{F}_{k}(\cdot;\bm{v}^t)$, starting from $\bm{w}_{k,0}^{t+1}=\bm{w}_{k}^{t}$, we have
	\begin{align}
	&\mathbb{E}\left[\tilde{F}_{k}(\bm{w}_{k,R}^{t+1}; \bm{v}^t)\right] 
	\le \tilde{F}_{k}(\bm{w}_k^t; \bm{v}^t) + \frac{\eta^2 R L_F \sigma^2}{2}  \nonumber \\
	&\quad - \eta R \left(1 - \frac{\eta L_F}{2}\right) 
	\cdot \frac{1}{R} \sum_{r=0}^{R-1} 
	\left\| \nabla \tilde{F}_{k}(\bm{w}_{k,r}^{t+1}; \bm{v}^t) \right\|^2.
	\end{align}
\end{lemma}

\begin{proof}
	This follows from the standard analysis of SGD on an $L_F$-smooth  function. 
	By smoothness, we have:
	\begin{align}
	\tilde{F}_{k}(\bm{w}_{k,r+1}^{t+1}) &\le \tilde{F}_{k}(\bm{w}_{k,r}^{t+1})  \nonumber \\
	&+ \langle \nabla\tilde{F}_{k}(\bm{w}_{k,r}^{t+1}), \bm{w}_{k,r+1}^{t+1} - \bm{w}_{k,r}^{t+1} \rangle  \nonumber \\
	&+ \frac{L_F}{2} \|\bm{w}_{k,r+1}^{t+1} - \bm{w}_{k,r}^{t+1}\|^2 \nonumber  \\
	&= \tilde{F}_{k}(\bm{w}_{k,r}^{t+1}) - \eta \langle \nabla \tilde{F}_{k}(\bm{w}_{k,r}^{t+1}), {\bm{d}}_{k,r} \rangle \nonumber  \\
	&+ \frac{\eta^2 L_F}{2} \|{\bm{d}}_{k,r}\|^2,
	\end{align}  
	where ${\bm{d}}_{k,r}=\nabla \hat{f}_{k}( \bm{w}_{k,r}^{t+1};\mathcal{B}_{k,r})+\lambda \nabla \tilde{g}(\bm{v}^{t},\bm{w}_{k,r}^{t+1})+\mu \bm{w}_{k,r}^{t+1}$ is the stochastic gradient.
	Using Assumption~\ref{assump: BSGV}, $\mathbb{E}[\|\bm{d}_{k,r}\|^2]\leq \|\nabla \tilde{F}_{k}\|^2+\sigma^2$.
	Substituting this in gives:
	\begin{align}
	\mathbb{E}[\tilde{F}_{k}(\bm{w}_{k,r+1}^{t+1})] \le& \tilde{F}_{k}(\bm{w}_{k,r}^{t+1}) - \eta \|\nabla \tilde{F}_{k}(\bm{w}_{k,r}^{t+1})\|^2 \notag\\
	&+ \frac{\eta^2 L_F}{2} \mathbb{E}[\|{\bm{d}}_{k,r}\|^2] \notag\\
	&\le \tilde{F}_{k}(\bm{w}_{k,r}^{t+1}) - \eta \|\nabla \tilde{F}_{k}(\bm{w}_{k,r}^{t+1})\|^2 \notag\\
	&+ \frac{\eta^2 L_F}{2} (\|\nabla \tilde{F}_{k}(\bm{w}_{k,r}^{t+1})\|^2 + \sigma^2) \notag\\
	=& \tilde{F}_{k}(\bm{w}_{k,r}^{t+1}) \notag\\
	&- \eta \left(1 - \frac{\eta L_F}{2}\right)\|\nabla \tilde{F}_{k}(\bm{w}_{k,r}^{t+1})\|^2 \notag\\
	&+ \frac{\eta^2 L_F \sigma^2}{2}.
	\end{align}
	
	Telescoping this inequality over $r=0,\ldots,R-1$ and taking the total expectation yields the desired result.
\end{proof}

\paragraph{B.3. Main Convergence Proof (Theorem 1).}
We now present our main convergence result by analyzing the evolution of a carefully chosen Lyapunov function. 
In our personalized federated setting, a single global model does not exist. 
Instead, our goal is to show that the distributed optimization process is stable and converges to a meaningful equilibrium.
To this end, we analyze the evolution of a global smoothed potential function 
\begin{equation}
\Psi^t= \sum_{k=1}^{K} p_k \tilde{F}_{k}(\bm{w}_{k}^{t};\bm{v}^t).
\end{equation}

This function tracks the state of the entire distributed system, encompassing all personalized models $\bm{w}_{k}^{t}$ and the server's consensus vector $\bm{v}^t$.

\begin{theorem}[Convergence to a Stationary Neighborhood]
	Let assumptions \ref{assump: L-smoothness}-\ref{assump: BTGV} hold.
	For a learning rate  $\eta \leq \frac{1}{L_F}$,  after $T$ rounds of Algorithm 1 where $S$ out of $K$ clients are sampled each round, we have
	\begin{align}
	&\frac{1}{T}\sum_{t=0}^{T-1}\frac{1}{R}\sum_{r=0}^{R-1} \mathbb{E}\left[\sum_{k=1}^K p_k \|\nabla \tilde{F}_{k}(\bm{w}_{k,r}^{t+1}; \bm{v}^t)\|_2^2 \right] \nonumber \\
	&\le \frac{\Psi^0 - F^*}{c_1 T} + \frac{\eta^2 R L_F \sigma^2}{2c_1} + \frac{\Delta_{\max}}{c_1}+\frac{\lambda E_S}{c_1},
	\end{align}
	where $c_1=\eta R(1-\eta L_F/2)$.
	The error terms are defined as follows:
	\begin{itemize}
		\item $\Delta_{\max}=2\lambda(\sqrt{m}C_\Phi W + m)$ bounds the error from one-bit quantization.
		\item $E_S$ bounds the error from client sampling and is given by:
		\textcolor{black}{
		\begin{align}
		E_S =  2m \sqrt{\frac{K(K-S)}{S(K-1)} \sum_{k=1}^K p_k^2}.
		\end{align}}
	\end{itemize}
\end{theorem}

\begin{remark}
	Note that the client sampling error $E_S$ vanishes when $S=K$ (full client participation). 
	In this case, our convergence bound recovers the result for the full participation setting.
\end{remark}

\begin{proof}
	We analyze the one-round change in the potential function, $\Psi^{t+1} - \Psi^t$. 
	We decompose this change as:
	\begin{align}
	&\mathbb{E}[\Psi^{t+1} - \Psi^t] \nonumber \\ 
	&= \mathbb{E}\bigg[\underbrace{\sum_{k=1}^K p_k (\tilde{F}_{k}(\bm{w}_k^{t+1}; \bm{v}^{t+1}) - \tilde{F}_{k}(\bm{w}_k^{t+1}; \bm{v}^t))}_{=:\Gamma_1}\bigg] \nonumber \\
	&+ \mathbb{E}\bigg[\underbrace{\sum_{k=1}^K p_k(\tilde{F}_{k}(\bm{w}_k^{t+1}; \bm{v}^t) - \tilde{F}_{k}(\bm{w}_k^t; \bm{v}^t))}_{=:\Gamma_2} \bigg].
	\end{align}
	
	{\bf{(1) Bounding the Client Progress.}}
	From Lemma~\ref{lemma:client_update}, after taking expectations and weighting by $p_k$, the client progress is bounded by:
	\begin{align}
	\mathbb{E}[\Gamma_2] &=\mathbb{E}\left[\sum_{k=1}^K p_k \left(\tilde{F}_{k}(\bm{w}_k^{t+1}; \bm{v}^t) - \tilde{F}_{k}(\bm{w}_k^t; \bm{v}^t)\right)\right] \nonumber \\
	&\le -c_1 \sum_{k=1}^K p_k \left(\frac{1}{R}\sum_{r=0}^{R-1} \|\nabla \tilde{F}_{k}(\bm{w}_{k,r}^{t+1}; \bm{v}^t)\|^2 \right) \nonumber  \\
	& + \frac{\eta^2 R L_F \sigma^2}{2},
	\end{align}
	
	\noindent where $c_1=\eta R(1-\frac{\eta L_F}{2})$.
	
	{\bf{(2) Bounding the Server Progress.}}
	The server progress term reflects the change in the potential function due solely to the server updating its vector from $\bm{v}^{t}$ to $\bm{v}^{t+1}$.
	Based on the definition of $\tilde{F}_{k}$, we define:
	\begin{align}
	\mathbb{E}[\Gamma_1] &=\mathbb{E}\bigg[\sum_{k=1}^K p_k \left( \tilde{F}_{k}(\bm{w}_k^{t+1}; \bm{v}^{t+1}) - \tilde{F}_{k}(\bm{w}_k^{t+1}; \bm{v}^t) \right) \bigg] \nonumber \\
	&= \lambda \mathbb{E} \bigg [\sum_{k=1}^K p_k \left( \langle \bm{v}^t, \mathbf{\Phi}\bm{w}_k^{t+1} \rangle - \langle \bm{v}^{t+1}, \mathbf{\Phi}\bm{w}_k^{t+1} \rangle \right) \bigg] \nonumber  \\
	&= \lambda \mathbb{E}\Big[\Big\langle \bm{v}^t - \bm{v}^{t+1}, \sum_{k=1}^K p_k \mathbf{\Phi}\bm{w}_k^{t+1} \Big\rangle \Big] \nonumber  \\
	&= K \lambda \mathbb{E}\Big[\Big\langle \bm{v}^t - \bm{v}^{t+1}, \frac{1}{K} \sum_{k=1}^K p_k \mathbf{\Phi}\bm{w}_k^{t+1} \Big\rangle \Big]
	\end{align}
	
	Here, we must acknowledge the mismatch: $\tilde{F}_{k}$ uses $\mathbf{\Phi}\bm{w}$, but the server update rule $\bm{v}^{t+1}$  uses $\mbox{sign}(\mathbf{\Phi}\bm{w})$. 
	We bridge this gap by introducing the quantization error $\bm{\delta}_k^{t+1}=\mathbf{\Phi}\bm{w}_k^{t+1}-\mbox{sign}(\mathbf{\Phi}\bm{w}_k^{t+1})$.
	Let $\bm{z}_k^{t+1} = \mathrm{sign}(\mathbf{\Phi}\bm{w}_k^{t+1})$ and $\bm{\delta}_k^{t+1}$ be the quantization error. 
	Let $\hat{\bm{z}}^{t+1} = \frac{1}{S}\sum_{k \in \mathcal{S}^t} p_k \bm{z}_k^{t+1}$ be the sampled aggregate and $\bar{\bm{z}}^{t+1} = \frac{1}{K}\sum_{k=1}^K p_k \bm{z}_k^{t+1}$ be the ideal aggregate. 
	The server update is $\bm{v}^{t+1} = \mathrm{sign}(\hat{\bm{z}}^{t+1})$.
	We decompose the inner product by introducing the sampled aggregate $\hat{\bm{z}}^{t+1}$ and the quantization error:
	\begin{align}
	&\frac{\mathbb{E}[\Gamma_1]}{ \lambda} \nonumber \\ &= K \mathbb{E}\left[\left\langle \bm{v}^t - \bm{v}^{t+1}, \frac{1}{K}\sum_{k=1}^K p_k \bm{z}_k^{t+1} + \frac{1}{K}\sum_{k=1}^K p_k \bm{\delta}_k^{t+1} \right\rangle\right] \nonumber \\
	&= \mathbb{E}[\underbrace{K \left\langle \bm{v}^t - \bm{v}^{t+1}, \hat{\bm{z}}^{t+1} \right\rangle}_{=:\Theta_1}] \nonumber \\
	&+ \mathbb{E}[\underbrace{K \left\langle \bm{v}^t - \bm{v}^{t+1}, \bar{\bm{z}}^{t+1} - \hat{\bm{z}}^{t+1} \right\rangle}_{=:\Theta_2}] \nonumber \\
	&+ \mathbb{E}\bigg[\underbrace{\bigg\langle \bm{v}^t - \bm{v}^{t+1}, \sum_{k=1}^K p_k \bm{\delta}_k^{t+1} \bigg \rangle}_{=:\Theta_3}\bigg].
	\end{align}
	
	Now we bound each term:
	\begin{itemize}
		\item \textbf{Term} $\mathbf{\Theta_1}$\textbf{:} From Lemma \ref{lemma: server agg}, we know $\bm{v}^{t+1}$ is the optimal solution for the objective based on the sampled sketches $\hat{\bm{z}}^{t+1}$. 
		This means $\langle \bm{v}^{t+1}, \hat{\bm{z}}^{t+1} \rangle \ge \langle \bm{v}^{t}, \hat{\bm{z}}^{t+1} \rangle$. Therefore, Term A is non-positive: $\mathbb{E}[\Theta_1] \le 0$.
		
		\item \textbf{Term} $\mathbf{\Theta_2}$\textbf{:} We use the Cauchy-Schwarz inequality, the fact that $\|\bm{v}^t - \bm{v}^{t+1}\|_2 \le 2\sqrt{m}$, and our Lemma \ref{lem:sampling_variance}. 
		\textcolor{black}{
		\begin{align}
		\mathbb{E}[\Theta_2] &\le K \mathbb{E}[\|\bm{v}^t - \bm{v}^{t+1}\|_2 \cdot \|\bar{\bm{z}}^{t+1} - \hat{\bm{z}}^{t+1}\|_2]  \nonumber \\
		&\le 2\sqrt{m} \sqrt{\mathbb{E}[\|\bar{\bm{z}}^{t+1} - \hat{\bm{z}}^{t+1}\|_2^2]}  \nonumber \\
		&\le 2\sqrt{m} \sqrt{\frac{K(K-S)}{S(K-1)} \sum_{k=1}^K \left\| p_k\bm{z}_k^{t+1} - \bar{\bm{z}}^{t+1} \right\|^2_2} \nonumber\\
		&\le 2\sqrt{m} \sqrt{\frac{K(K-S)}{S(K-1)}\sum_{k=1}^K \left\| p_k\bm{z}_k^{t+1}\right\|^2_2} \nonumber\\
		&\le 2\sqrt{m} \sqrt{\frac{K(K-S)}{S(K-1)}\sum_{k=1}^K  p_k^2\left\|\bm{z}_k^{t+1}\right\|^2_2}  \nonumber\\
		&\le 2m \sqrt{\frac{K(K-S)}{S(K-1)} \sum_{k=1}^K p_k^2}.
		\end{align}
		Here, we utilize the variance bound property $\sum_{k=1}^K \| \bm{x}_k - \bar{\bm{x}} \|^2 \le \sum_{k=1}^K \| \bm{x}_k \|^2$. Letting $\bm{x}_k = p_k\bm{z}_k^{t+1}$ and using $\|\bm{z}_k\|_2^2 = m$, the term is bounded by $m \sum_{k=1}^K p_k^2$.}
		
		\item \textbf{Term} $\mathbf{\Theta_3}$\textbf{:} Since the entries of $\bm{v}^t$ and $\bm{v}^{t+1}$ are $\pm 1$, the entries of their difference are in $\{-2,0,2\}$.
		This gives:
		\begin{align}
		\mathbb{E}[\Theta_3] &\le \mathbb{E}\left[ \|\bm{v}^t - \bm{v}^{t+1}\|_\infty \left\| \sum_{k=1}^K p_k \bm{\delta}_k^{t+1} \right\|_1 \right]  \nonumber \\
		&\le 2 \cdot \mathbb{E}\left[ \left\| \sum_{k=1}^K p_k \bm{\delta}_k^{t+1} \right\|_1 \right] \quad  \nonumber \\
		&(\text{since} \|\bm{v}^t - \bm{v}^{t+1}\|_{\infty} \le 2)  \nonumber \\
		&\le 2 \sum_{k=1}^K p_k \mathbb{E}\left[\|\bm{\delta}_k^{t+1}\|_1 \right].
		\end{align}
		As established in Lemma \ref{lemma:Bounded Model Norm}, the expected squared norm of the weights $\mathbb{E}[\|\bm{w}_k^{t+1}\|_2^2]$is uniformly bounded by a constant $W$. 
		This allows us to bound the expected quantization error for each client:
		\begin{align}
		\mathbb{E}[\|\bm{\delta}_k^{t+1}\|_1] &= \mathbb{E}[\|\mathbf{\Phi}\bm{w}_k^{t+1} - \mathrm{sign}(\mathbf{\Phi}\bm{w}_k^{t+1})\|_1]  \nonumber \\
		&\le \mathbb{E}[\|\mathbf{\Phi}\bm{w}_k^{t+1}\|_1] + m  \nonumber \\
		&\le \sqrt{m}\mathbb{E}[\|\mathbf{\Phi}\bm{w}_k^{t+1}\|_2] + m  \nonumber \\
		&\le \sqrt{m}\sqrt{\mathbb{E}[\|\mathbf{\Phi}\bm{w}_k^{t+1}\|_2^2]} + m  \nonumber \\
		&\le \sqrt{m}\sqrt{C_\Phi^2 \mathbb{E}[\|\bm{w}_k^{t+1}\|_2^2]} + m  \nonumber \\
		&\le \sqrt{m}C_\Phi W + m.
		\end{align}
		Then we have
		\begin{equation}
		\mathbb{E}[\Theta_3] \le 2(\sqrt{m}C_\Phi W + m).
		\end{equation}
		We define $\Delta_{\max} \triangleq 2\lambda(\sqrt{m}C_\Phi W + m)$, so that this term is bounded by $\Delta_{\max} / \lambda$.
	\end{itemize}
	
	\textcolor{black}{
	Combining these bounds, the total expected server progress is bounded by:
	\begin{align}
	\mathbb{E}[\Gamma_1]  \le & 2\lambda m \sqrt{\frac{K(K-S)}{S(K-1)} \sum_{k=1}^K p_k^2} + 2\lambda(\sqrt{m}C_\Phi W + m).
	\end{align}}
	
	{\bf{(3) Combining and Telescoping.}}
	Now, combining the bounds for the expected client and server progress, we get the one-round recurrence:
	\textcolor{black}{
	\begin{align}
	&\mathbb{E}[\Psi^{t+1}]  \nonumber \\
	&\le \mathbb{E}[\Psi^t]- c_1 \mathbb{E}\left[\sum_{k=1}^K p_k \frac{1}{R}\sum_{r=0}^{R-1} \|\nabla \tilde{F}_{k}(\bm{w}_{k,r}^{t+1}; \bm{v}^t)\|_2^2\right] \nonumber \\
	&+ \frac{\eta^2 R L_F \sigma^2}{2} + \Delta_{\max} + \lambda \left( 2m \sqrt{\frac{K(K-S)}{S(K-1)} \sum_{k=1}^K p_k^2} \right),
	\end{align}}
	\noindent where $c_1\triangleq \eta R(1-\eta L_F/2)$.
	
	Rearranging and summing over $t=0,\ldots,T-1$ :
	\textcolor{black}{
	\begin{align}
	&\sum_{t=0}^{T-1} c_1 \mathbb{E}\left[\sum_{k=1}^K p_k \frac{1}{R}\sum_{r=0}^{R-1} \|\nabla \tilde{F}_{k}(\bm{w}_{k,r}^{t+1}; \bm{v}^t)\|_2^2\right] \nonumber \\
	&\le \mathbb{E}[\Psi^0] - \mathbb{E}[\Psi^T] + T\left(\frac{\eta^2 R L_F \sigma^2}{2} + \Delta_{\max}\right) \nonumber \\
	&+ \lambda T \left( 2m \sqrt{\frac{K(K-S)}{S(K-1)} \sum_{k=1}^K p_k^2} \right).
	\end{align}}
	
	Using the fact that $\Psi^T \geq F^*$(Assumption~\ref{assump: Bounded Below}) and dividing by $c_1 T$, we arrive at the final result:
	\textcolor{black}{
	\begin{align}
	&\frac{1}{T}\sum_{t=0}^{T-1}\mathbb{E}\left[\sum_{k=1}^K p_k \frac{1}{R}\sum_{r=0}^{R-1} \|\nabla \tilde{F}_{k}(\bm{w}_{k,r}^{t+1}; \bm{v}^t)\|_2^2\right] \nonumber \\
	& \le \frac{\Psi^0 - F^*}{c_1 T} + \frac{\eta^2 R L_F \sigma^2}{2c_1} + \frac{\Delta_{\max}}{c_1} \nonumber \\
	&+ \frac{2m \lambda}{c_1}  \sqrt{\frac{K(K-S)}{S(K-1)} \sum_{k=1}^K p_k^2} ,
	\end{align}}
	which completes the proof.
\end{proof}

\end{document}